%% file: main_icml.tex
\documentclass{article}

\usepackage{microtype}
\usepackage{graphicx}
\usepackage{subfigure}
\usepackage{booktabs} 

\usepackage{hyperref}

\usepackage[accepted]{icml2024}

\usepackage{amsmath}
\usepackage{amssymb}
\usepackage{mathtools}
\usepackage{amsthm}
\usepackage{physics}

\usepackage[capitalize,noabbrev]{cleveref}

\theoremstyle{plain}
\newtheorem{theorem}{Theorem}[section]

\newtheorem{lemma}[theorem]{Lemma}
\newtheorem{corollary}[theorem]{Corollary}
\theoremstyle{definition}

\theoremstyle{remark}
\newtheorem{remark}[theorem]{Remark}
\usepackage{enumitem}

\usepackage{thmtools}
\usepackage{thm-restate}


\def\eqref#1{(\ref{#1})}

\DeclareMathAlphabet\mathbfcal{OMS}{cmsy}{b}{n}

\usepackage{hyperref}
\usepackage{cleveref}
\usepackage{url}

\usepackage{ifthen}

\newboolean{print_impact_statement}
\setboolean{print_impact_statement}{true}
\newboolean{print_acknowledgements}
\setboolean{print_acknowledgements}{true}

\icmltitlerunning{Emergent Equivariance in Deep Ensembles}

\begin{document}

\twocolumn[
\icmltitle{Emergent Equivariance in Deep Ensembles}



\icmlsetsymbol{equal}{*}

\begin{icmlauthorlist}
\icmlauthor{Jan E. Gerken}{equal,chalmers}
\icmlauthor{Pan Kessel}{equal,prescientdesign}
\end{icmlauthorlist}

\icmlaffiliation{chalmers}{Department of Mathematical Sciences, Chalmers University of Technology and the University of Gothenburg, SE-412 96 Gothenburg, Sweden}
\icmlaffiliation{prescientdesign}{Prescient Design, Genentech Roche, Basel, Switzerland}

\icmlcorrespondingauthor{Jan Gerken}{gerken@chalmers.se}
\icmlcorrespondingauthor{Pan Kessel}{pan.kessel@roche.com}

\icmlkeywords{Machine Learning, ICML}

\vskip 0.3in
]



\printAffiliationsAndNotice{\icmlEqualContribution} 

\input{body_icml}

\onecolumn

\include{supp_mat_icml}

\end{document}

%% file: body_icml.tex
\begin{abstract}
  We show that deep ensembles become equivariant for all inputs and at all training times by simply using full data augmentation. Crucially, equivariance holds off-manifold and for any architecture in the infinite width limit. The equivariance is emergent in the sense that predictions of individual ensemble members are not equivariant but their collective prediction is. Neural tangent kernel theory is used to derive this result and we verify our theoretical insights using detailed numerical experiments.
\end{abstract}

\section{Introduction}
Deep ensembles are a standard workhorse of deep learning practitioners \cite{lakshminarayanan2017simple}. They operate by averaging the prediction of several networks and therefore offer a straightforward way of estimating the uncertainty  of the prediction. For example, deep ensembles are widely used in the medical domain such as in cancer cell detection in pathology or in protein folding in drug design as quantifying the confidence of the output is critical in these fields~\cite{saib2020hierarchical, ruffolo2023fast}.

The main message of this paper is that deep ensembles offer a novel and straightforward way to enforce equivariance with respect to symmetries of the data. Specifically, we show that upon full data augmentation, deep ensembles become equivariant \emph{at all training steps and for any input} in the large width limit. While this statement would be trivial for a fully trained model and on the data manifold, our results are significantly more powerful in that they also hold off-manifold and even at initialization. A deep ensemble is thus indistinguishable from a fully equivariant network. It is important to emphasize that this manifest equivariance is emergent: while the prediction of the ensemble is equivariant, the predictions of its members are not. In particular, the ensemble members are not required to have an equivariant architecture.

We rigorously derive this surprising emergent equivariance by using the duality between neural networks and kernel machines in the large width limit~\citep{neal1996,lee2018,yang2020a}. The neural tangent kernel (NTK) describes the evolution of deep neural networks during training \citep{jacot2018}. In the limit of infinite width, the neural tangent kernel is frozen, i.e., it does not evolve during training and the training dynamics can be solved analytically. As a random variable over initializations, the output of the neural network after arbitrary training time follows a Gaussian distribution whose mean and covariance are available as closed form expressions \citep{lee2019a}. In this context, deep ensembles can be interpreted as a Monte-Carlo estimate of the corresponding expected network output. This insight allows us to theoretically analyze the effect of data augmentation throughout training and show that the deep ensemble is fully equivariant.

In practice, this emergent equivariance of deep ensemble cannot be expected to hold perfectly and exact equivariance will be broken, since neural networks are not infinitely wide and the expectation value over initalizations is estimated by Monte-Carlo. Furthermore, in the case of a continuous symmetry group, data augmentation cannot cover the entire group orbit and is thus approximate. We analyze the resulting breaking of equivariance and demonstrate empirically that deep ensembles nevertheless show a competetively high degree of equivariance even with a low number of ensemble members.

The main contributions of our work are:\\[-2em]
\begin{itemize}
    \item We prove in our main theorem~\ref{th:main} that infinitely wide deep ensembles are equivariant at all stages of training and any input if trained with full data augmentation using the theory of neural tangent kernels.\\[-1.5em]
    \item We derive bounds for deviations from equivariance due to finite size as well as data augmentation for a continuous group.\\[-1.5em]
    \item We empirically demonstrate the emergent equivariance in three settings: Ising model, FashionMNIST, and a high-dimensional medical dataset of histological slices.
\end{itemize}

\section{Related Works}
\paragraph{Deep Ensembles, Equivariance.} There is a too extensive body of literature on both equivariance and deep ensembles to be summarized here, reflecting their central importance to modern deep learning. We refer to~\citet{gerken2023} and~\citet{ganaie2022} for reviews, respectively. The relation between manifest equivariance and data augmentation concerning model performance was studied by~\citet{gerken2022} and for training dynamics by~\citet{flinth2023}.
\paragraph{Equivariance without architecture constraints.} Equivariance can also be achieved by symmetrizing the network output over (an appropriately chosen subset of) the group orbit \cite{puny2021frame, basu2023equivariant, basu2023equi}. This approach is orthogonal to ours: instead of an ensemble of models, an ensemble of outputs is considered. Note that the memory footprint of the symmetrization depends on the size of the group orbit while, for deep ensembles, it depends on the number of ensemble members. Another architecture-agnostic method to reach equivariance is to homogenize inputs using a canonicalization network \cite{kaba2023equivariance, mondal2023equivariant}. Canonicalization and symmetrization lead to exact equivariance (up to possible discretization effects) while deep ensembles naturally allow for uncertainty estimation and increased robustness.
\paragraph{Neural Tangent Kernel.} That Bayesian neural networks behave as Gaussian processes was first discovered by~\citet{neal1996}, this result was extended to deep neural networks by~\citet{lee2018}. Neural tangent kernels (NTKs), which capture the evolution of wide neural networks under gradient descent training, were introduced by~\citet{jacot2018}. The literature on this topic has since expanded considerably so that we can only cite some selected works, a review on the topic is given by~\citet{golikov2022}. The NTK for CNNs was computed by~\citet{arora2019}. \citet{lee2019a} used the NTK to show that wide neural networks trained with gradient descent become Gaussian processes and \citet{yang2020a} introduced a comprehensive framework to study scaling limits of wide neural networks rigorously. This framework was used by~\citet{yang2022d} to find a parametrization suitable for scaling networks to large width. NTKs were used to study GANs~\cite{franceschi2022}, PINNs~\cite{wang2022b}, backdoor attacks~\cite{hayase2022} as well as pruning~\cite{yang2023a}, amongst other applications. Corrections to the infinite-width limit, in particular in connection to quantum field theory, have been investigated as well~\citep{huang2020,yaida2020,halverson2021,erbin2022}. 
\paragraph{Data augmentation and kernel machines.} \citet{mroueh2015}, \citet{raj2017} and \citet{mei2021} study properties of kernel machines using group-averaged kernels but they do not consider wide neural networks. \citet{dao2019} use a Markov process to model random data augmentations and show that an optimal Bayes classifier in this context becomes a kernel machine. It is also shown that training on augmented data is equivalent to using an augmented kernel. \citet{li2019a} introduce new forms of pooling to improve kernel machines. As part of their analysis, they derive the analogous augmented kernel results as \citet{dao2019} for the NTK at infinite training time. In contrast, we focus on the symmetry properties of the resulting (deep) ensemble of infinitely wide neural networks. In particular, we analyze the behavior of the ensemble at finite training time, show that their assumption of an ``equivariant kernel'' is satisfied under very mild assumptions on the representation (cf.\ Theorem~\ref{lemma1}), include equivariance on top of invariance and derive a bound for the invariance error accrued by approximating a continuous group with finitely many samples.

\section{Deep Ensembles and Neural Tangent Kernels}
In this section, we give a brief overview over deep ensembles and their connection to NTKs.

\vspace{-0.25em}
\paragraph{Deep Ensemble.} Let $f_w: X \to \mathbb{R}$ be a neural network with parameters $w$ which are initialized by sampling from the density $p$, i.e.\ $w \sim p$. For notational simplicity, we consider only scalar-valued networks in the main part of the paper unless stated otherwise. Our results however hold also for vector-valued networks. The output of the deep ensemble $\bar{f}_t$ of the network $f_w$ is then defined as the expected value over initializations of the trained ensemble members
\begin{align}
    \bar{f}_t(x) = \mathbb{E}_{w \sim p} \left[ f_{\mathcal{L}_t w}(x) \right] \,, \label{eq:deepensemble}
\end{align}
where the operator $\mathcal{L}_t$ maps the initial weight $w$ to its corresponding value after $t$ steps of gradient descent.
In practice, the deep ensemble is approximated by a Monte-Carlo estimate of the expectation value using a finite number $M$ of initializations
\begin{align}
    \bar{f}_t(x) \approx \hat{f}_t(x) =  \frac{1}{M} \sum_{i=1}^M f_{\mathcal{L}_t w_i}(x) \,, \label{eq:deepensemble_mc}
\end{align}
where $w_i \sim p$.
This amounts to performing $M$ training runs with different initializations and averaging the outputs of the resulting models. It is worthwhile to note that in the literature, the average $\hat{f}_t$ as defined in \eqref{eq:deepensemble_mc} is often referred to as the deep ensemble~\citep{lakshminarayanan2017simple}. In this work, we will however use the term deep ensemble to refer to the expectation value $\bar{f}_t$ of \eqref{eq:deepensemble}. Analogously, we refer to $\hat{f}_t$ as the MC estimate of the deep ensemble $\bar{f}_t$. 

\vspace{-0.25em}
\paragraph{Relation to NTK.}
In the infinite width limit, a deep ensemble follows a Gaussian distribution described by the neural tangent kernel \citep{jacot2018}
\begin{align}
    \Theta(x, x')  = \sum_{l=1}^L \mathbb{E}_{w \sim p} \left[ \left(\frac{\partial f_w(x)}{\partial w^{(l)}}\right)^\top \frac{\partial f_w(x')}{\partial w^{(l)}}\right]  \,,
    \label{eq:ntk_def}
\end{align}
where $w^{(l)}$ denotes the parameters of the $l$\textsuperscript{th} layer and we have assumed that the network has a total of $L$ layers. Here, the width is taken to infinity, resulting in Gaussian distributions, whose mean and covariance over the initialization distribution is then studied. In general, $\Theta$ has additional axes for dimensions not taken to infinity, e.g.\ pixels in CNNs and output channels in MLPs, which we will keep implicit in most of the main part. In the following, we use the notation
\begin{align}
\Theta_{ij} = \Theta(x_i, x_j)
\end{align}
for the Gram matrix, i.e.\ the kernel evaluated on two elements $x_i$ and $x_j$ of the training set
\begin{align}
 \mathcal{T} = (\mathcal{X}, \mathcal{Y}) = \{(x_i, y_i)\,|\,i=1,\dots,|\mathcal{T}|\} \,.
\end{align}
Using the NTK, we can analytically calculate the distribution of ensemble members in the large width limit for a given input $x$ at any training time $t$ for learning rate $\eta$: networks trained with the MSE loss follow a Gaussian process distribution with mean function $\mu_t$ and covariance function $\Sigma_t$ which are given in terms of the NTK by~\citet{lee2019a}
\begin{align}
 \mu_t(x) &= \Theta(x, x_i) \, \left[ \Theta^{-1} \, T_{t} \right]_{ij} \,y_j \,, \label{eq:output_mean}\\
\Sigma_t(x,x') &= \mathcal{K}(x, x') + \Sigma^{(1)}_t(x,x') - ( \Sigma^{(2)}_t(x,x') + \textrm{h.c.} ) \,,\label{eq:output_var}
\end{align}
where $T_t = (\mathbb{I} - \exp( - \eta \Theta t))$ and all sums over the training set are implicit by the Einstein summation convention and we have defined
\begin{align*}
    \Sigma_t^{(1)}(x,x') &= \Theta(x, x_i) \left[\Theta^{-1} \, T_t \; \mathcal{K}  T_t \, \Theta^{-1} \right]_{ij} \, \Theta(x_j, x') \,,\\
    \Sigma_t^{(2)}(x,x') &= \Theta(x, x_i) \, \left[ \Theta^{-1} \, T_t \right]_{ij} \, \mathcal{K}(x_j, x') \,,
\end{align*}
with the NNGP kernel
\begin{align}
\mathcal{K}(x, x')=\mathbb{E}_{w \sim p} \left[f_w(x) \, f_w(x') \right] \,.
\label{eq:nngp_def}
\end{align}
The Gram matrix of the NNGP is given by $\mathcal{K}_{ij}=\mathcal{K}(x_i,x_j)$. For $\Sigma_t$ in a less compact notation, see~\eqref{eq:cov_fct_long} in Appendix~\ref{app:ntk-intro}.

\begin{remark}
The function $\mu_t$ in \eqref{eq:output_mean} captures the mean output of networks trained on different initializations for time $t$. Therefore, it is just the expected ensemble output \eqref{eq:deepensemble}, $\bar{f}_t(x)=\mu_t(x)$. The variance of this quantity is given by the covariance function evaluated at identical arguments $\Sigma_t(x):=\Sigma_t(x,x)$.
\end{remark}

In Appendix~\ref{app:ntk-intro} we provide a brief review of NTK theory for readers unfamiliar with it.

In practice, the cost of inverting the Gram matrix is prohibitive. Therefore, one typically estimates the deep ensemble by \eqref{eq:deepensemble_mc} using $M$ trained models with different random initalizations. Nevertheless, the dual NTK description allows us to reason about the properties of the exact deep ensemble. In the following, we will use this duality to theoretically investigate the effect of data augmentation on deep ensembles.

\section{Equivariance and Data Augmentation}
In this section, we summarize basics facts about representations of groups, equivariance, and data augmentation and establish our notation.

\paragraph{Representations of Groups.} Groups abstractly describe symmetry transformations. In order to describe how a group transforms a vector, we use group representations. A (linear) representation of a group $G$ is a map $\rho: G \to \textrm{GL}(V)$ where $V$ is a vector space and $\rho$ is a group homomorphism, i.e.\ $\rho(g_1) \rho(g_2) = \rho(g_1 g_2)$ for all $g_1, g_2 \in G$. A representation is called orthogonal if $\rho(g^{-1}) = \rho(g)^\top$, i.e., if it has orthogonal representation matrices. 

\vspace{-0.25em}
\paragraph{Equivariance.} For learning tasks in which data $x$ and labels $y$ transform under group representations, the map $x\mapsto y$ has to be compatible with the symmetry group; this property is called equivariance. Formally, let $f: X \to Y$ denote a (possibly vector valued) model with input space $X$ and output space $Y$ on which the group $G$ acts with representations $\rho_X$ and $\rho_Y$, respectively. Then, $f$ is equivariant with respect to the representations $\rho_X$ and $\rho_Y$ if it obeys
\begin{align}
    \rho_Y(g) f(x) = f( \rho_X(g) x)   \qquad \forall x \in X, g \in G \,.\label{eq:def_equiv}
\end{align}
Similarly, a model $f$ is invariant with respect to the representation $\rho_X$ if it satisfies the above relation with $\rho_Y$ being the trivial representation, i.e. $\rho_Y(g)=\mathbb{I}$ for all $g \in G$.
Considerable work has been done to construct manifestly equivariant neural networks with respect to specific, practically important special cases of \eqref{eq:def_equiv}. It has been shown both empirically (e.g.\ in \citet{thomas2018}, \citet{bekkers2018}) and theoretically (e.g.\ in \citet{sannai2021}, \citet{elesedy2021a}) that equivariance can lead to better sample efficiency, improved training speed and greater robustness. A downside of equivariant architectures is that they need to be purpose-built for symmetry properties of the problem at hand since standard well-established architectures are mostly not equivariant.

\vspace{-0.25em}
\paragraph{Data Augmentation.} An alternative approach to incorporate information about the symmetries of the data into the model is data augmentation. Instead of using the original training set $\mathcal{T}$, we use a set which is augmented by all elements of the group orbit, i.e.
\begin{align}
\mathcal{T}_{\mathrm{aug}} = \{ (\rho_X(g) x, \rho_Y(g) y) | g \in G, (x,y) \in \mathcal{T} \} \,.
\end{align}
In stochastic gradient descent, we randomly draw a minibatch from this augmented training set to estimate the gradient of the loss.
If the group has finite order, data augmentation has the immediate consequence that the action of any group element $g \in G$ on a training sample can be written as a permutation $\pi_g$ of the indices of the augmented training set $\mathcal{T}_{\mathrm{aug}}$, i.e.
\begin{align}
    \rho_X(g) x_i = x_{\pi_g(i)}   \qquad \textrm{and} \qquad \rho_Y(g) y_i = y_{\pi_g(i)} \,, \label{eq:perm_group_action}
\end{align}
where $i \in \{1, \dots, |\mathcal{T}_{\mathrm{aug}}|\}$. 
Data augmentation has the advantage that it does not impose any restrictions on the architecture and is hence straightforward to implement. However, the symmetry is only learned and it can thus be expected that the model is only (approximately) equivariant towards the end of training and on the data manifold. Furthermore, the model cannot benefit from the restricted function space which the symmetry constraint specifies.

\begin{figure*}[tb]
  \centering
  \includegraphics[width=0.9\textwidth]{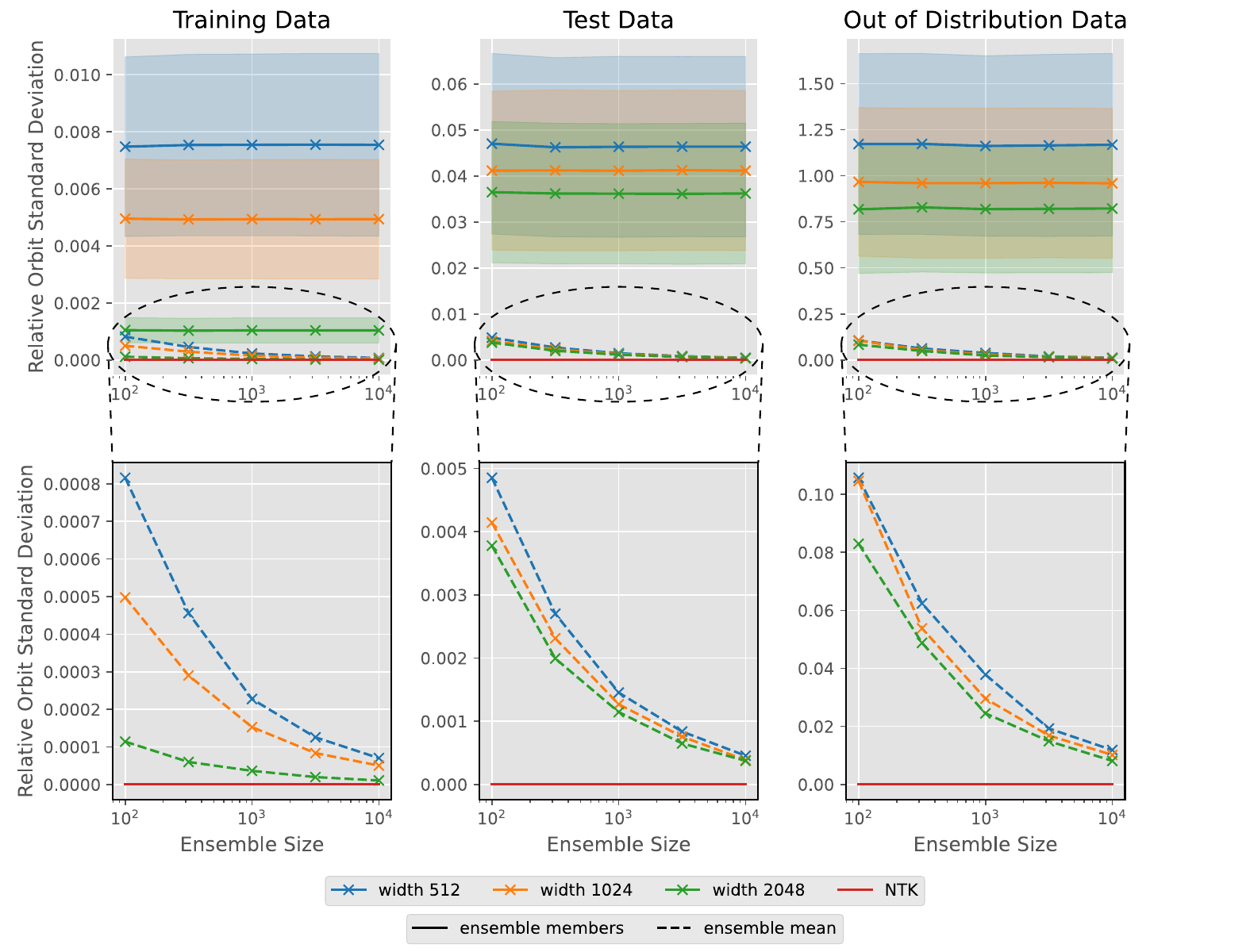}
  \caption{Invariance of predicted energies with respect to lattice rotations by $90^{\circ}$. Solid lines refer to predictions of individual ensemble members and their standard deviation, dashed lines refer to mean predictions of the ensemble. Zoom-ins in the second row show that the invariance of mean predictions converges to NTK invariance for large ensembles and network widths.}
  \label{fig:orbit_stds}
\end{figure*}

\section{Emergent Equivariance for Large-Width Deep Ensembles}
In this section, we prove that any large-width deep ensemble is emergently equivariant when data augmentation is used. After stating our assumptions, the sketch the proof in three steps.  

\vspace{-0.5em}
\paragraph{Assumptions.}  We consider a finite group $G$ with representations $\rho_X$ and $\rho_Y$ as well as data augmentation with respect to these representations, as discussed above. The case of continuous groups will be discussed subsequently. If the input or output have spatial axes $a$, the representations $\rho_X$ and $\rho_Y$ act via a representation $\rho$ on that domain,
\begin{align}
    \rho_X(g)x_i^a &= \tau_X(g)x^{\rho^{-1}(g)a}\label{eq:rhoX}\\
    \rho_Y(g)y_i^a &= \tau_Y(g)y^{\rho^{-1}(g)a}\,.\label{eq:rhoY}
\end{align}
The representations $\tau_{X,Y}$ are assumed to be orthogonal and act on the channel dimensions of the input and output. E.g.\ for rotations on images in the input, $\tau_X=\mathbb{I}$ and $\rho$ is the fundamental representation of $SO(2)$ in terms of $2\times 2$ rotation matrices. For graph neural networks, we consider orthogonal transformations of the node features, $\rho_X(g)x^v=\tau_X(g)x^v$ with node index $v$. Hence, in this case $\rho=\mathbb{I}$ is trivial. Our results hold for fairly general architectures consisting of convoluational, fully-connected, and flattening layers as well as local aggregation layers in graph neural networks trained on the MSE loss. To illustrate the underlying techniques of the proof, we will prove each step using the simple example of a MLP with a single channel dimension before stating the general results derived in the appendix.

\paragraph{Step 1:} 
The representation $\rho_X$ acting on the input space $X$ induces a canonical transformation of the NTK and NNGP kernel
\begin{align}
    \Theta(x, x') \;\;  &\to \;\; \Theta(\rho_X(g) x, \rho_X(g) x') \label{eq:ntk_trafo} \\
    \mathcal{K}(x, x') \;\;  &\to \;\; \mathcal{K}(\rho_X(g) x, \rho_X(g) x') \label{eq:nngp_trafo}\,.
\end{align}  
For a representation $\rho_X$ acting on the input space $X$, this canonical transformation induces a transformation of the output indices as specified by the following theorem:
\begin{restatable}[Kernel transformation]{theorem}{firstlemma}
\label{lemma1}
Let $G$ be a group and $\rho_X$ a representation of $G$ acting on the input space $X$ as in \eqref{eq:rhoX}. Then, the neural tangent kernel $\Theta$, as defined in~\eqref{eq:ntk_def}, as well as the NNGP kernel $\mathcal{K}$, as defined in~\eqref{eq:nngp_def}, of a neural network satisfying the assumptions above transform according to
\begin{align}
    \Theta(\rho_X(g) x, \rho_X(g) x') = \rho_{K}(g)\Theta (x, x')\rho_{K}^{\top}(g) \,, \\
    \mathcal{K}(\rho_X(g) x, \rho_X(g) x') = \rho_{K}(g)\mathcal{K}(x, x')\rho_{K}^{\top}(g) \,,
\end{align}
for all $g \in G$ and $x,x' \in X$, where $\rho_{K}$ is a transformation acting on the spatial dimensions of the kernels according to $\rho_{K}(g) K^{a}=K^{\rho^{-1}(g)a}$. If the kernels do not have spatial axes, $\rho_K=\mathbb{I}$.
\end{restatable}
\begin{proof}\let\qed\relax
    See Appendix~\ref{app:proofs}.
\end{proof}

Note that Theorem~\ref{lemma1} states in particular that MLP-kernels are invariant since they do not have spatial axes. While this kernel invariance is shared by many standard kernels, such as RBF or linear kernels, this property is non-trivial for NTK and NNGP since they are not simply functions of the norm of the difference or inner product of the two input values $x$ and $x'$. Furthermore, this result holds irrespective of whether a group is of finite or infinite order.

\paragraph{Step 2:} Data augmentation allows to rewrite the group action as a permutation (see~\eqref{eq:perm_group_action}). For the Gram matrix, acting with $\pi_g$ is equivalent to multiplication by a permutation matrix $\Pi(g)$. Combining this with the invariance of the MLP-kernels derived above, we can shift a permutation from the first to the second index of the Gram matrix, i.e., for MLPs,
\begin{align}
  \Pi(g)\Theta(\mathcal{X},\mathcal{X})&=\Theta(\rho_{X}(g)\mathcal{X},\mathcal{X})\\
  &=\Theta(\mathcal{X},\rho_{X}^{-1}(g)\mathcal{X})\\
  &=\Theta(\mathcal{X},\mathcal{X})(\Pi^{-1}(g))^{\top}\\
  &=\Theta(\mathcal{X},\mathcal{X})\Pi(g)\,,\label{eq:shift_permutation_ntk}
\end{align}%
where we have used the Theorem~\ref{lemma1} for the second equality. This property can be extended to more general architectures and analytical functions of the kernels as stated in the following lemma.

\begin{restatable}[Shift of permutation]{lemma}{secondlemma}
\label{lemma2}
Data augmentation implies that the permutation group action $\Pi$ commutes with any matrix-valued analytical function $F$ involving the Gram matrices of the NNGP and NTK as well as their inverses:
\begin{align}
  &\Pi(g)F(\Theta, \Theta^{-1}, \mathcal{K}, \mathcal{K}^{-1}) \nonumber\\
  =&\, \rho_{K}(g)F(\Theta, \Theta^{-1}, \mathcal{K},\mathcal{K}^{-1})\Pi(g)\rho_{K}^{\top}(g) \,.
\end{align}
where $\Pi(g)$ denotes the group action in terms of training set permutations as a permutation matrix, see \eqref{eq:perm_group_action}.

\end{restatable}
\begin{proof}\let\qed\relax
    See Appendix~\ref{app:proofs}.
\end{proof}

\paragraph{Step 3:} Using Lemma~\ref{lemma2}, it can be shown that the deep ensemble is equivariant in the infinite width limit. Before stating the general theorem, we first illustrate the underlying reasoning by showing one particular consequence, i.e., that the mean of an MLP is invariant if the training labels $\mathcal{Y}$ are not transformed, i.e.\ $\rho_Y(g)=\mathbb{I}$. By \eqref{eq:output_mean}, the output of the deep ensemble for transformed input $x \to \rho_X(g) x$ is given by
\begin{align}
    \bar{f}_t(\rho_X(g) \, x)&=\mu_t(\rho_X(g) \, x) \\
    &= \Theta(\rho_X(g) \, x, \mathcal{X}) \left[ \Theta^{-1} \, T_{t} \right]\mathcal{Y}\,.
\end{align}
We can now use the invariance of MLP kernels from Theorem~\ref{lemma1} and write the action of $\rho_{X}$ on training samples $\mathcal{X}$ as a permutation $\Pi$. Together with $(\Pi^{-1})^{\top}=\Pi$, we obtain
\begin{align}
    \Theta(\rho_X(g) \, x, \mathcal{X}) \left[ \Theta^{-1} T_{t} \right] \mathcal{Y}=
    \Theta(x, \mathcal{X})\Pi(g) \left[ \Theta^{-1} T_{t} \right]\mathcal{Y}\,.\nonumber
\end{align}
Now we use Lemma~\ref{lemma2} to commute the permutation past $\Theta^{-1}T_{t}$,
\begin{align}
    \Theta(x, \mathcal{X})\Pi(g) \left[ \Theta^{-1} T_{t} \right]\mathcal{Y}=\Theta(x, \mathcal{X})\left[ \Theta^{-1} T_{t} \right]\Pi(g) \mathcal{Y}\,.\nonumber
\end{align}
Since the labels are invariant by assumption, $\Pi(g)\mathcal{Y}=\mathcal{Y}$ and therefore
\begin{align}
    \bar{f}_t(\rho_X(g) \, x) &= \bar{f}_t(x) \,.
\end{align}%
Using analogous reasoning, the following more general result can be derived:
\begin{restatable}[Emergent Equivariance of Deep Ensembles]{theorem}{thmemergingequiv}\label{th:main}
Under the assumptions stated above, the distribution of large-width ensemble members $f_w: X \to Y$ is equivariant with respect to the representations $\rho_X$ and $\rho_Y$ of the group $G$ if data augmentation is applied. In particular, the ensemble is equivariant,
\begin{align}
    \bar{f}_t(\rho_X(g) \, x) = \rho_Y(g) \, \bar{f}_t(x) \,,
\end{align}
for all $g \in G$. This result holds
\begin{enumerate}
    \item at any training time $t$,
    \item for any element of the input space $x \in X$.
\end{enumerate}
\end{restatable}
\begin{proof}\let\qed\relax
See Appendix~\ref{app:proofs}.
\end{proof}
We stress that this results holds even off the data manifold, i.e., for out-of-distribution data, and in the early stages of training as well as at initialization. As a result, it is not a trivial consequence of the training. Furthermore, we do not need to make any restrictions on the architectures of the ensemble members. In particular, the individual members will generically not be equivariant. However, their averaged prediction will be (at least in the large width limit). In this sense, the equivariance is emergent.
\begin{figure*}[tb]
  \centering
  \includegraphics[width=0.48\linewidth]{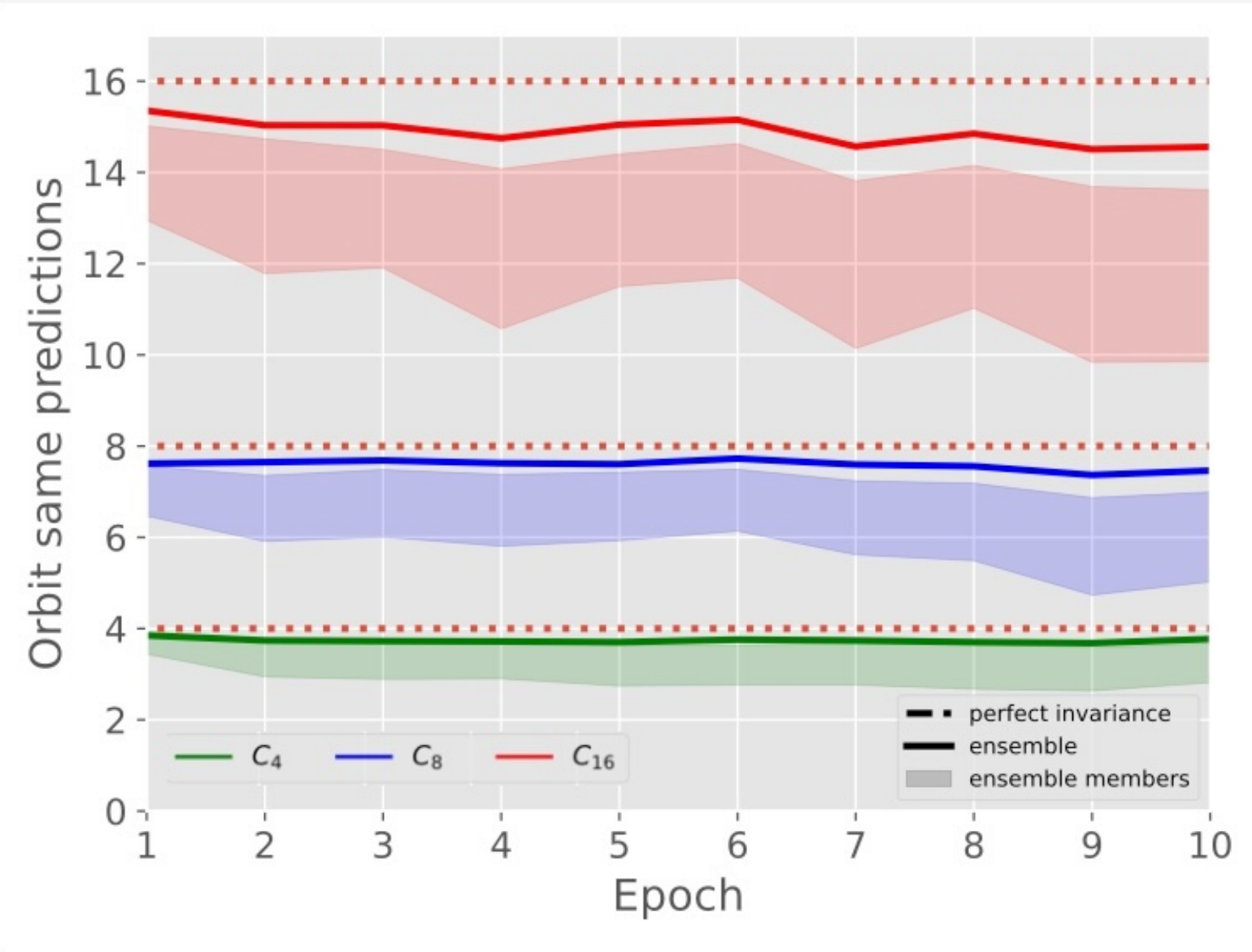}
  \hfill
  \includegraphics[width=0.48\linewidth]{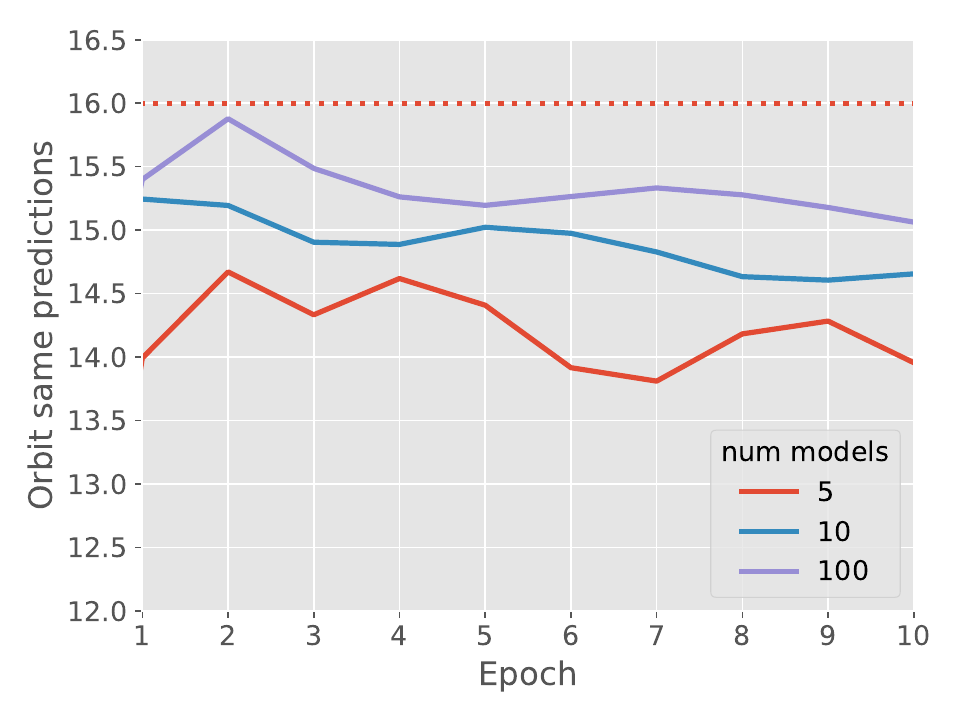}
  \caption{Emergent invariance for FashionMNIST \textbf{Left:} Number of out-of-distribution MNIST samples with the same prediction across a symmetry orbit for group orders 4 (green), 8 (blue), and 16 (red) versus training epoch. The models were trained on augmented FashionMNIST. Solid lines show the ensemble prediction. Shaded area is between the 25\textsuperscript{th} and 75\textsuperscript{th} quantile of the predictions of individual members of the ensemble. \textbf{Right:} Out of distribution invariance in the same setup as on the left-hand-side at group order 16. As the number of ensemble members increases, the prediction becomes more invariant, as expected.}
  \label{fig:fmnist}
\end{figure*}

\begin{figure}[tb]
  \centering
  \includegraphics[width=0.98\linewidth]{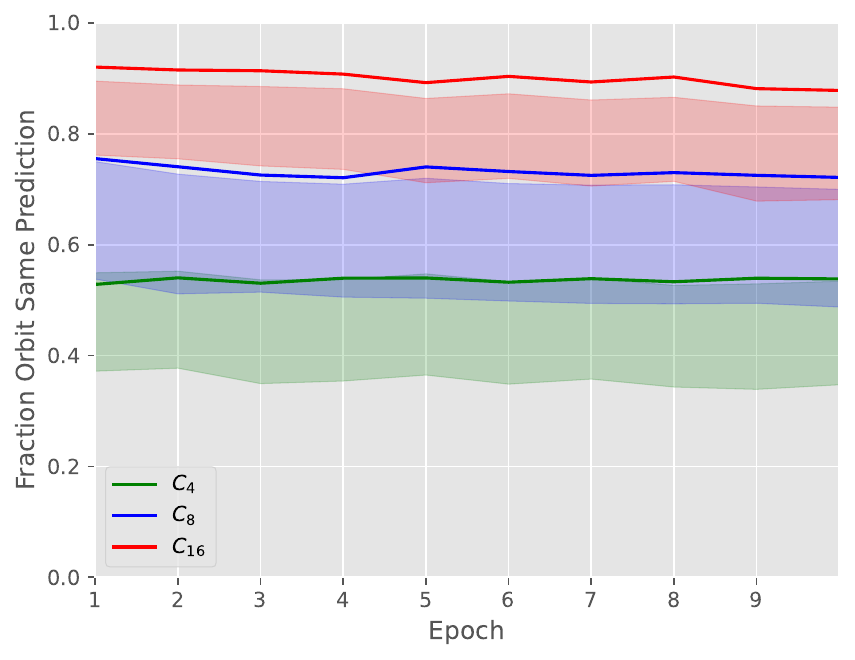}
  \caption{Equivariance extends to $SO(2)$ symmetry. Fraction of randomly sampled rotations that leave the prediction invariant is reported. Data augmentation with group order 4 (green), 8 (blue), 16 (red) is used. As expected, the equivariance increases with the group order.}
  \label{fig:so2}
\end{figure}

\begin{figure*}[tb]
  \centering
  \includegraphics[width=0.48\linewidth]{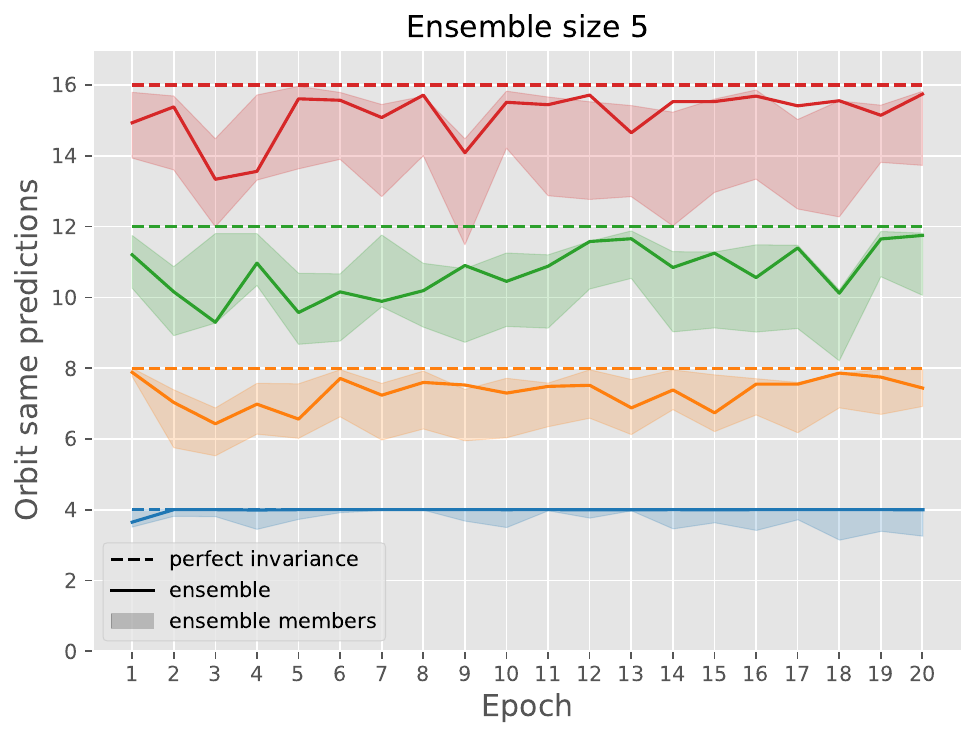}
  \hfill
  \includegraphics[width=0.48\linewidth]{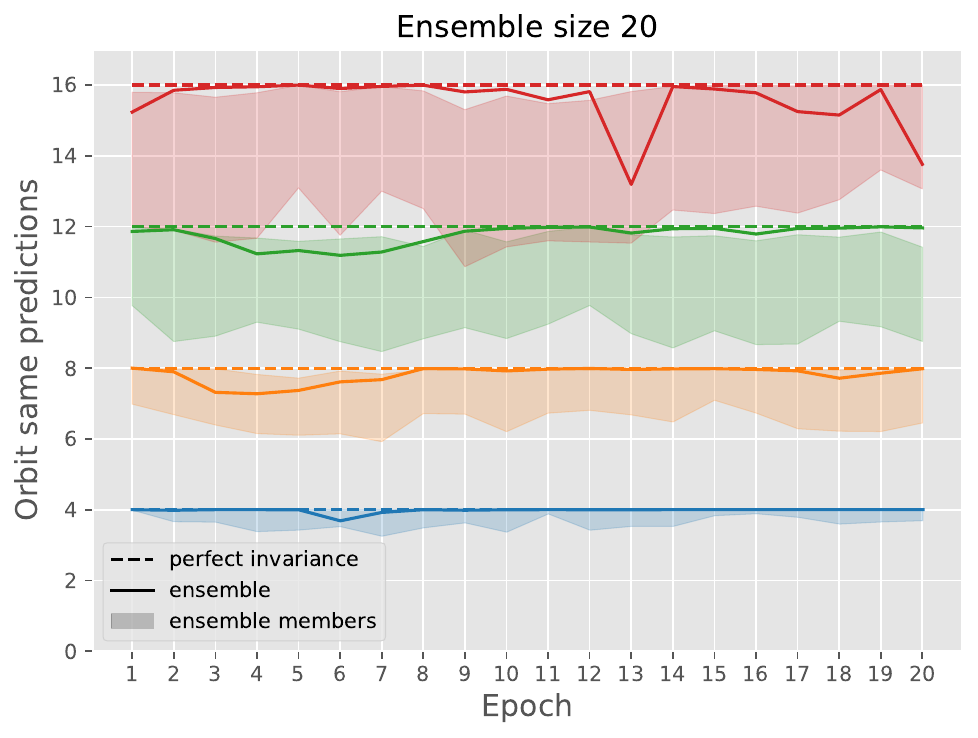}
  \caption{Ensemble invariance on OOD data for ensembles trained on histological data. Number of OOD samples with the same prediction across a symmetry orbit for group orders 4 (blue), 8 (orange), 12 (green) and 16 (red) versus training epoch. Even for ensemble size 5 (left), the ensemble predictions (solid line) are more invariant than the ensemble members (shaded region corresponding to 25\textsuperscript{th} to 75\textsuperscript{th} percentile of ensemble members). The effect is larger for ensemble size 20 (right).}
  \label{fig:hist_ood_plots}
\end{figure*}

\section{Limitations: Approximate Equivariance}
In the following, we discuss the breaking of equivariance due to i) statistical fluctuations of the estimator due to the finite number of ensemble members, ii) continuous symmetry groups which do not allow for complete data augmentation, and iii) finite width corrections in NTK theory.

\vspace{-0.25em}
\paragraph{Finite Number of Ensemble Members.}
We derive the following bound for estimates of deep ensembles in the infinite width limit:
\begin{restatable}[Bound for finite ensemble members]{lemma}
{lemboundfiniteensemble}
    The deep ensemble $\bar{f}_t$ and its estimate $\hat{f}_t$ do not differ by more than threshold $\delta$, 
    \begin{align}
        | \bar{f}_t(x) - \hat{f}_t(x) | < \delta \,,
    \end{align}
    with probability $1 - \epsilon$ for ensemble sizes $M$ that obey
    \begin{align}
       M >  - \frac{2 \Sigma_t(x)}{\delta^2} \ln \left( \sqrt{\pi}  \epsilon \right)\,.
    \end{align}
\end{restatable}
\vspace{-0.5em}
We stress that the covariance $\Sigma$ is known in closed form, see \eqref{eq:output_var}. As such, the right-hand-side can be calculated exactly. We note that we also derive a somewhat tighter bound in Appendix~\ref{app:proofs} which however necessitates to numerically solve for $M$.

\vspace{-0.25em}
\paragraph{Continuous Groups.}
For a continuous group $G$, consider a finite subgroup $A \subset G$ which is used for data augmentation. We quantify the discretization error of using $A$ instead of $G$ by
\begin{align}
    \epsilon = \max_{g \in G} \, \min_{g'\in A} \|\rho_X(g) - \rho_X(g')\|_{\mathrm{op}} \,.\label{eq:disc_error}
\end{align}
Then, the invariance error of the mean \eqref{eq:output_mean} is bounded by $\epsilon$:
\begin{restatable}[Bound for continuous groups]{lemma}{thboundcont}\label{th:continuous}
Consider a deep ensemble of neural networks with Lipschitz continuous derivatives with respect to the parameters. For an approximation $A\subset G$ of a continuous symmetry group G with discretization error $\epsilon$, the prediction of the ensemble trained on $A$ deviates from invariance by
\begin{align*}
    |\bar{f}_t(x) - \bar{f}_t(\rho_X(g) \, x)| \le \epsilon \, C(x)\,,\qquad\forall g\in G\,,
\end{align*}
where $C$ is independent of $g$.
\end{restatable}

\paragraph{Random Augmentations.}
In practice, very often the augmentation is not performed over an entire subgroup $A$ of the symmetry group as assumed in Lemma~\ref{th:continuous}, but rather batches are augmented randomly. That is, $A$ is not a subgroup, only a subset of $G$. In this case, the error \eqref{eq:disc_error} of using $A$ rather than $G$ for augmentation can be defined in terms of an expectation value over the distribution of the augmentations. The statement of Lemma~\ref{th:continuous} can then only be expected to hold in expectation. However, note that the solution of the training dynamics derived using NTKs in the infinite width limit assumes that the training set is the same in each epoch. Normally this assumption will be broken by random data augmentation. This effect cannot be controlled by Lemma~\ref{th:continuous}.

\paragraph{Finite Width.}
Convergence of the ensemble output to a Gaussian distribution only holds in the infinite width limit. There has been substantial work on finite-width corrections to the NTK limit~\citep{huang2020,yaida2020,halverson2021,erbin2022} which could in principle be used to quantify the resulting violations of exact equivariance. This is however of significant technical difficulty and therefore beyond the scope of this work. In the experimental section, we nevertheless demonstrate that even finite-width ensembles show emergent equivariance to good approximation.

\section{Experiments}
In this section, we empirically study the emergent equivariance of finite width deep ensembles for several architectures (fully connected and convolutional), tasks (regression and classification), and application domains (computer vision and physics). 

\paragraph{Ising Model.}

We validate our analytical computations with experiments on a problem for which we can compute the NTK exactly: the two-dimensional Ising model on a 5x5 lattice with energy function $\mathcal{E} = -J \sum_{\langle i, j \rangle} s_i s_j$, with the spins $s_{i}\in\{+1,-1\}$, $J$ a coupling constant and the sum runs over all adjacent spins. 
The energy of the Ising model is invariant under the cyclic group $C_4$ of rotations of the lattice by $90^{\circ}$. We train ensembles of five different sizes with $100$ to $10\mathrm{k}$ members of fully-connected networks with hidden-layer widths $512$, $1024$ and $2048$ to approximate the energy function using all rotations in $C_4$ as data augmentation. In this setting, we can compute the NTK exactly on the given training data using the JAX package \texttt{neural-tangents}~\citep{neuraltangents2020}. We verify that the ensembles converge to the NTK for large widths, see Appendix~\ref{app:ising_model}.

To quantify the invariance of the ensembles, we measure the standard deviation of the predicted energy across the group orbit averaged over all datapoints of i) training set, ii) test set, and iii) out-of-distribution set. The latter is generated randomly drawing spins from a Gaussian distribution with mean zero and variance $400$. For better interpretability, we divide by the mean of $\mathcal{E}$, so that for a \emph{relative standard deviation (RSD)} across orbits of one, the deviation from invariance is as large as a typical ground truth energy. For an exactly equivariant model, we would obtain an RSD of zero.

Figure~\ref{fig:orbit_stds} shows that the deep ensemble indeed exhibit the expected emergent invariance. 
As expected, the NTK features very low RSD compatible with numerical error. The RSD of the mean predictions of the ensembles are larger but still very small and converge to the NTK results for large ensembles and network widths, cf.\ dashed lines in Figure~\ref{fig:orbit_stds}. In contrast, the RSD computed for individual ensemble members is much higher and varies considerably between ensemble members, cf.\ solid lines in Figure~\ref{fig:orbit_stds}. Even out of distribution, the ensemble means deviate from invariance only by about $0.8\%$ for large ensembles and network widths, compared to $82\%$ for individual ensemble members.

\begin{table}[h!]
\centering
\begin{tabular}{c c c c}
 \toprule
 & $C_4$ & $C_8$ & $C_{16}$  \\ [0.5ex] 
 \midrule
\hspace{-0.1cm}DeepEns+DA\hspace{-0.1cm} & 3.85$\pm$0.12  & \bf{7.72$\pm$0.34} & \bf{15.24$\pm$0.69} \\
 only DA & 3.41$\pm$0.18 &6.73$\pm$0.24 & 12.77$\pm$0.71 \\ 
 E2CNN & 
\bf{4$\pm$0.0} & \bf{7.71$\pm$0.21} & \bf{15.08$\pm$0.34} \\
 Canon & \bf{4$\pm$0.0} & \bf{7.45$\pm$0.14} & 12.41$\pm$0.85  \\
 \bottomrule
\end{tabular}
\caption{Deep Ensembles (ensemble size 100) show competitive degree of equivariance. Mean and standard deviation over training of orbit same prediction on the out-of-distribution MNIST validation set. All methods use roughly the same number of parameters. Best methods, taking into account statistical uncertainty, are shown in bold.}
\label{table:method_compare}
\end{table}

\paragraph{FashionMNIST.} We train convolutional neural networks augmenting the original dataset~\citep{xiao2017/online} by all elements of the group orbit of the cyclic group $C_k$, i.e., all rotations of the image by any multiple of $360/k$ degrees with $k=4, 8, 16$ and choose ensembles of size $M=5, 10, 100$. We then evaluate the \emph{orbit same prediction (OSP)}, i.e., how many of the images in a given group orbit have on average the same classification result as the unrotated image. We evaluate the OSP metric both on the validation set of FashionMNIST as well as on various out-of-distribution (OOD) datasets. Specifically, we choose the validation sets of MNIST, grey-scaled CIFAR10 as well as images for which each pixel is drawn iid from $\mathcal{N}(0,1)$. Figure~\ref{fig:fmnist} shows the OSP metric for OOD data from MNIST. The ensemble prediction becomes more invariant as the number of ensemble members increases. Furthermore, the ensemble prediction is significantly more invariant as the individual ensemble members, i.e., the invariance is emergent. As the group order $k$ increases, more ensemble members are needed to achieve a high degree of invariance. Figure~\ref{fig:so2} illustrates that the deep ensembles can also capture continuous rotation symmetry. Specifically, we train using data augmentation with respect to various discrete $C_k$ groups and check that they lead to increasing invariance with respect to $SO(2)$. For $k=16$, over 90 percent of the orbit elements have the same prediction as the untransformed input establishing that the model is approximately invariant under the continuous symmetry as well, as is expected from Lemma~\ref{th:continuous}. We also compare to a manifestly equivariant E2CNN \cite{cesa2022a, e2cnn} and canonicalized model \cite{kaba2023equivariance}. Interestingly, the manifest equivariance of these models is slightly broken for groups $C_k$ with group order $k>4$ due to interpolation artifacts. As result, finite deep ensembles are competitive with these manifestly equivariant models, see Table~\ref{table:method_compare}. Note also that using data augmentation without any ensembling leads to significantly less equivariant models. More details about the experiments as well as plots showing results for the other OOD datasets can be found in Appendix~\ref{app:fmnist}.

\paragraph{Histological Data.}

A realistic task, where rotational invariance is of key importance, is the classification of histological slices. We trained ensembles of CNNs on the NCT-CRC-HE-100K dataset~\citep{hist_dataset} which comprises of stained histological images of human colorectal cancer and normal tissue with a resolution of $224\times 224$ pixels in nine classes.

As for our experiments on FashionMNIST, we verify that the ensemble is more invariant as a function of its input than the ensemble members by evaluating the OSP on OOD data. In order to arrive at a sample of OOD data on which the network makes non-constant predictions, we optimize the input of the untrained ensemble to yield balanced predictions of high confidence. Using this specifically generated dataset for each ensemble, we observe the same increase in invariance also outside of the training domain as predicted by our theoretical considerations, cf.\ Figure~\ref{fig:hist_ood_plots}. For further results on validation data as well as examples of our OOD data see Appendix~\ref{app:hist_data}

\section{Conclusions}
Equivariant neural networks are a central ingredient in many machine learning setups, in particular in the natural sciences. However, constructing manifestly invariant models can be difficult. Deep ensembles are an important tool which can straightforwardly boost the performance and estimate uncertainty of existing models, explaining their widespread use in practice. 

In this work, using the theory of neural tangent kernels, we proved that infinitely wide ensembles show emergent equivariance when trained on augmented data. We furthermore discussed implications of finite width and ensemble size as well as the effect of approximating a continuous symmetry group. Experiments on several different datasets support our theoretical insights.

The extension of our proof to additional layers like attention, pooling or dropout is straightforward. In future work, it would be interesting to incorporate the effects of finite width corrections and include a more detailed model of data augmentation, for instance along the lines of~\citet{dao2019}.

\ifthenelse{\boolean{print_impact_statement}}{
\section*{Impact statement}
This paper presents work whose goal is to advance the field of Machine Learning. There are many potential societal consequences of our work, none which we feel must be specifically highlighted here.
}
{}

\ifthenelse{\boolean{print_acknowledgements}}{
\section*{Acknowledgements}
The work of J.G. is supported by the Wallenberg AI, Autonomous Systems and Software Program (WASP) funded by the Knut and Alice Wallenberg Foundation. P.K. wants to thank Shin Nakajima and Andreas Loukas for discussions, J.G. wants to thank Max Guillen and Philipp Misof for discussions.
}
{}

\small

\bibliography{main}
\bibliographystyle{icml2024}

%% file: supp_mat_icml.tex
\appendix

\section{Introduction to neural tangent kernels}\label{app:ntk-intro}
In this appendix, we will give a brief review of the theory of neural tangent kernels (NTKs) for readers who are not familiar with it. For a more comprehensive review, see e.g.~\citet{golikov2022}.

\subsection{The empirical NTK}
To understand how the NTK arises in the training dynamics of neural networks, consider a neural network $f_{w}:X\rightarrow\mathbb{R}$ with parameters $w$. Under continuous gradient descent, the parameters are updated according to
\begin{align}
  \frac{\partial w}{\partial t}=-\eta \frac{\partial \mathcal{L}(f_{w}(\mathcal{X}),\mathcal{Y})}{\partial w}\,, \label{eq:3}
\end{align}
where $t$ is the training time, $\eta$ is the learning rate and $\mathcal{L}(f_{w}(\mathcal{X}),\mathcal{Y})$ is the loss function which depends on the training set predictions $f_{w}(\mathcal{X})$ and the training labels $\mathcal{Y}$. Since $\mathcal{L}$ depends on $w$ only through $f_{w}(\mathcal{X})$, we can use the chain-rule to rewrite \eqref{eq:3},
\begin{align}
  \frac{\partial w}{\partial t}=-\eta \sum_{i=1}^{|\mathcal{T}|}\frac{\partial f_{w}(x_{i})}{\partial w}\frac{\partial\mathcal{L}(f_{w}(\mathcal{X}),\mathcal{Y})}{\partial f_{w}(x_{i})}\,.\label{eq:4}
\end{align}
Similarly, $f_{w}$ depends on $t$ only through $w$, so we obtain
\begin{align}
  \frac{\partial f_{w}(x)}{\partial t}=\left(\frac{\partial f_{w}(x)}{\partial w}\right)^{\top}\frac{\partial w}{\partial t}\,.
\end{align}
And hence, using \eqref{eq:4},
\begin{align}
  \frac{\partial f_{w}(x)}{\partial t}=-\eta \sum_{i=1}^{|\mathcal{T}|}\left(\frac{\partial f_{w}(x)}{\partial w}\right)^{\top}\frac{\partial f_{w}(x_{i})}{\partial w}\frac{\partial\mathcal{L}(f_{w}(\mathcal{X}),\mathcal{Y})}{\partial f_{w}(x_{i})}
  =-\eta\sum_{i=1}^{|\mathcal{T}|}\Theta^{w}(x,x_{i})\frac{\partial\mathcal{L}(f_{w}(\mathcal{X}),\mathcal{Y})}{\partial f_{w}(x_{i})}\,,\label{eq:5}
\end{align}
where we have introduced the \emph{empirical neural tangent kernel}
\begin{align}
  \Theta^{w}(x,x')=\left(\frac{\partial f_{w}(x)}{\partial w}\right)^{\top}\frac{\partial f_{w}(x')}{\partial w}\,.\label{eq:6}
\end{align}
This quantity depends on the parameters and hence evolves during training. As can be seen in \eqref{eq:5}, it is the NTK which induces the complicated non-linear evolution in the training dynamics since the derivative of the loss with respect to the predictions is independent of the parameters. Since it depends on the parameters, we can think of the empirical NTK at initialization as a random variable over the initialization distribution.

\subsection{Infinite width limit}

At infinite width, the dynamics \eqref{eq:5} simplify dramatically. Firstly, it is known for a long time that the preactivations at initialization become a mean-zero Gaussian process (GP) as the width of the layers tend to infinity~\cite{neal1996}. The covariance function of this GP is known as the \emph{neural network gaussian process (NNGP) kernel} $\mathcal{K}(x,x')$. Therefore, in particular
\begin{align}
  f_{w_{0}}(x)\sim\mathcal{N}(0,\mathcal{K}(x,x))
  \qquad \forall x \in X
  \label{eq:7}
\end{align}
at initialization. The NNGP kernel can be computed recursively layer-by-layer with the recursion relations given in the proof of Theorem~\ref{lemma1} in Appendix~\ref{app:proofs} below.

Secondly, it was realized more recently~\cite{jacot2018} that in the infinite width limit, the empirical NTK \eqref{eq:6} converges in probability to its expectation value and therefore becomes a deterministic quantity
\begin{align}
  \Theta^{w}(x,x')\xrightarrow{\text{width}\rightarrow\infty}\mathbb{E}_{w}\left[\left(\frac{\partial f_{w}(x)}{\partial w}\right)^{\top}\frac{\partial f_{w}(x')}{\partial w}\right]=\Theta(x,x')\,,\label{eq:9}
\end{align}
which is the definition of the NTK we used above in \eqref{eq:ntk_def}. This limiting quantity can be computed again recursively layer-by-layer.

In~\cite{jacot2018}, the authors introduced a slightly different parametrization of neural network layers. For fully connected layers of input width $n$, instead of using
\begin{align}
  z^{(\ell)}(x)=W f^{(\ell)}(x) \qquad \text{with} \qquad W_{ij}\sim\mathcal{N}\Big(0,\frac{1}{n}\Big)
\end{align}
they suggest to use
\begin{align}
  z^{(\ell)}(x)=\frac{1}{\sqrt{n}}W f^{(\ell)}(x) \qquad \text{with} \qquad W_{ij}\sim\mathcal{N}(0,1)\,.\label{eq:8}
\end{align}
Note that at initialization, the distribution of $z^{(\ell)}(x)$ is the same in both parametrizations. During training, the derivatives with respect to $W$ will however be rescaled in \eqref{eq:8}. An important result of \cite{jacot2018} was that under mild assumptions on the nonlinearity, in the infinite width limit the NTK does not only become a deterministic, but also constant throughout training when using the parametrization~\eqref{eq:8}.

\subsection{Training dynamics}
In the parametrization \eqref{eq:8}, the training dynamics \eqref{eq:5} therefore simplify dramatically when taking the layer widths to infinity. For the MSE loss, \eqref{eq:5} becomes
\begin{align}
  \frac{\partial f_{t}(x)}{\partial t}=-\eta\sum_{i=1}^{|\mathcal{T}|}\Theta(x,x_{i})(f_{t}(x_{i})-y_{i})=-\eta\ \Theta(x,\mathcal{X})(f_{t}(\mathcal{X})-\mathcal{Y})\,,\label{eq:10}
\end{align}
where we have introduced the notation $f_{t}$ for the neural network with parameters from training time $t$, $\Theta$ is the deterministic and constant NTK \eqref{eq:9} and we have used the compact notation for summations over the training data employed in many places of this paper.

The differential equation \eqref{eq:10} can be solved analytically in two steps. Since the right-hand side depends on $f_{t}$ evaluated on the training set, whereas the left-hand side depends on $f_{t}$ evaluated at an arbitrary point, we first solve \eqref{eq:10} evaluated on the training set,
\begin{align}
  \frac{\partial f_{t}(\mathcal{X})}{\partial t}=-\eta\ \Theta(\mathcal{X},\mathcal{X})(f_{t}(\mathcal{X})-\mathcal{Y})\,.
\end{align}
The solution to this equation is, in terms of the initial training-set predictions $f_{0}(\mathcal{X})$, given by
\begin{align}
  f_{t}(\mathcal{X})=e^{-\eta\Theta(\mathcal{X},\mathcal{X})t}(f_{0}(\mathcal{X})-\mathcal{Y})+\mathcal{Y}\,.
\end{align}
Next, we can plug this solution back into \eqref{eq:10} to obtain an equation for $f_{t}(x)$,
\begin{align}
  \frac{\partial f_{t}(x)}{\partial t}=-\eta\ \Theta(x,\mathcal{X})e^{-\eta\Theta(\mathcal{X},\mathcal{X})t}(f_{0}(\mathcal{X})-\mathcal{Y})\,.
\end{align}
The right-hand side depends on $t$ only through the exponential factor. Integration therefore yields a solution for $f_{t}(x)$ which we write in terms of the initial prediction $f_{0}(x)$
\begin{align}
  f_{t}(x)=\Theta(x,\mathcal{X})\Theta(\mathcal{X},\mathcal{X})^{-1}(e^{-\eta\Theta(\mathcal{X},\mathcal{X})t}-\mathbb{I})(f_{0}(\mathcal{X})-\mathcal{Y})+f_{0}(x)\,.\label{eq:11}
\end{align}
This completely solves the training dynamics and we can predict the output at an arbitrary test point after an arbitrary amount of training time.

The prediction at time $t$ is still a random variable of the initialization distribution. However, the initialization only enters via $f_{0}$ which in the infinite width limit is a GP as noted above. Therefore, $f_{t}$ is as a linear combination of GPs itself a GP. Since the mean function of $f_{0}$ is identically zero \eqref{eq:7}, it is straightforward to compute the mean function $\mu_{t}$ of $f_{t}$
\begin{align}
  \mu_{t}(x)=\mathbb{E}[f_{t}(x)]=\Theta(x,\mathcal{X})\Theta(\mathcal{X},\mathcal{X})^{-1}(\mathbb{I}-e^{-\eta\Theta(\mathcal{X},\mathcal{X})t})\mathcal{Y}\,.
\end{align}
This is just the expression given above in \eqref{eq:output_mean}. Similarly, one can compute the covariance function $\Sigma_{t}$ of $f_{t}$
\begin{align}
  \Sigma_{t}(x,x')&=\mathbb{E}[(f_{t}(x)-\mu_{t}(x))(f_{t}(x')-\mu_{t}(x'))]\\
  &=\mathbb{E}\left[\left( \Theta(x,\mathcal{X})\Theta(\mathcal{X},\mathcal{X})^{-1}(e^{-\eta\Theta(\mathcal{X},\mathcal{X})t}-\mathbb{I})f_{0}(\mathcal{X})+f_{0}(x) \right)\right.\nonumber\\
  &\quad\left.\times\left( f_{0}(\mathcal{X})^{\top} (e^{-\eta\Theta(\mathcal{X},\mathcal{X})t}-\mathbb{I})\Theta(\mathcal{X},\mathcal{X})^{-1}\Theta(\mathcal{X},x')+f_{0}(x') \right)\right]\\
  &=\mathcal{K}(x,x')-\mathcal{K}(x,\mathcal{X})\Theta(\mathcal{X},\mathcal{X})^{-1}(\mathbb{I}-e^{-\eta\Theta(\mathcal{X},\mathcal{X})t})\Theta(\mathcal{X},x')\nonumber\\
  &\quad-\Theta(x,\mathcal{X})\Theta(\mathcal{X},\mathcal{X})^{-1}(\mathbb{I}-e^{-\eta\Theta(\mathcal{X},\mathcal{X})t})\mathcal{K}(\mathcal{X},x')\nonumber\\
  &\quad+\Theta(x,\mathcal{X})\Theta(\mathcal{X},\mathcal{X})^{-1}(\mathbb{I}-e^{-\eta\Theta(\mathcal{X},\mathcal{X})t})\mathcal{K}(\mathcal{X},\mathcal{X})(\mathbb{I}-e^{-\eta\Theta(\mathcal{X},\mathcal{X})t})\Theta(\mathcal{X},\mathcal{X})^{-1}\Theta(\mathcal{X},x')\,, \label{eq:cov_fct_long}
\end{align}
where we have used that the expectation value of $f_{0}$ vanishes and that the covariance function of $f_{0}$ is the NNGP $\mathcal{K}$. The final expression is \eqref{eq:output_var} from above. Since the predictions on arbitary test points are a GP, by providing explicit expressions for the mean- and convariance functions, we have determined the statistics of the predictions entirely.

\section{Proofs}\label{app:proofs}
\firstlemma*
\begin{proof}

We will prove the transformation properties by induction over the layer, using the forward equations for the kernels of the different layer types considered. In this prescription, both NNGP and NTK are defined recursively per layer. In the first layer, they are simply given by
 \begin{align}
    \mathcal{K}_1^{a,a'}(x,x')&=x_{a}^\top x'_{a'}\\
    \Theta_1^{a,a'}(x,x')&=\mathcal{K}_1^{a,a'}(x,x')\label{eq:cntk_1}\,,
 \end{align}
 where $a$ and $a'$ both denote a spatial axis, e.g.\ in two dimensions, $a$ is a multi-index $a=(h,w)$ and for graphs, $a$ is the node index. The $a,a'$ axes can be absent for inputs without spatial axes.

The forward equations which account for the nonlinearities are given by
\begin{align}
    \Lambda_{\ell}^{a,a'}(x,x')&=\begin{pmatrix}
        \mathcal{K}_\ell^{a,a}(x,x) & \mathcal{K}_\ell^{a,a'}(x,x')\\
        \mathcal{K}_\ell^{a',a}(x',x) & \mathcal{K}_\ell^{a',a'}(x',x')
    \end{pmatrix}\label{eq:lambeda_forward}\\
    \mathcal{K}_{\ell}^{a,a'}(x,x')&=\mathbb{E}_{(u,v)\sim\mathcal{N}(0,\Lambda_{\ell-1}^{a,a'}(x,x'))}[\sigma(u)\sigma(v)]\label{eq:nngp_forward_nonlin}\\
    \dot{\mathcal{K}}_{\ell}^{a,a'}(x,x')&=\mathbb{E}_{(u,v)\sim\mathcal{N}(0,\Lambda_{\ell-1}^{a,a'}(x,x'))}[\sigma'(u)\sigma'(v)]\\
    \Theta^{a,a'}_{\ell}(x,x')&=\dot{\mathcal{K}}_{\ell}^{a,a'}(x,x')\Theta^{a,a'}_{\ell-1}(x,x')\,,
\end{align}
where $\sigma$ is the nonlinearity and $\sigma'$ its derivative. Note that throughout, we drop numerical prefactors which depend on the prefactors in the layer definitions and initialization variances and are irrelevant for our argument.

For fully connected layers, the forward equation for the kernels is given by
\begin{align}
    \mathcal{K}_{\ell}(x,x')&=\mathcal{K}_{\ell-1}(x,x')\label{eq:sigma_forward_dense}\\ 
    \Theta_\ell(x,x')&=\mathcal{K}_\ell(x,x')+\Theta_{\ell-1}(x,x')\,.\label{eq:ntk_forward_dense}
\end{align}
For convolutional layers, the forward equation for the kernels is given by~\cite{arora2019}
\begin{align}
    \mathcal{K}^{a,a'}_{\ell}(x,x')&=\sum_{\tilde{a}}\mathcal{K}_{\ell-1}^{a+\tilde{a},a'+\tilde{a}}(x,x')\\ 
    \Theta_{\ell}^{a,a'}(x,x')&=\mathcal{K}^{a,a'}_{\ell}(x,x')+\sum_{\tilde{a}}\Theta_{\ell-1}^{a+\tilde{a},a'+\tilde{a}}(x,x')\,.
\end{align}
For flattening layers, the forward equation for the kernels is given by~\cite{neuraltangents2020}
\begin{align}
    \mathcal{K}^{a,a'}_{\ell}(x,x')&=\sum_{\tilde{a}}\mathcal{K}^{\tilde{a},\tilde{a}}_{\ell-1}(x,x')\label{eq:sigma_forward_flattening}\\
    \Theta_\ell(x,x')&=\mathcal{K}^{a,a'}_{\ell}(x,x')+\sum_{\tilde{a}}\Theta^{\tilde{a},\tilde{a}}_{\ell-1}(x,x')\,.
\end{align}

In graph neural networks we consider graphs with node features $x_a\in\mathbb{R}^n$ at node $a$. In a local aggregation layer, we sum the node features over a neighborhood $\mathcal{N}(a)$ of $a$,
\begin{align}
    z_{\ell}^{a}(x)=\sum_{\tilde{a}\in\mathcal{N}(a)}z_{\ell-1}^{\tilde{a}}(x)\,.
\end{align}
For local aggregation layers, the forward equations are given by~\cite{du2019}
\begin{align}
    \mathcal{K}^{a,a'}_{\ell}(x,x')&=\sum_{\tilde{a}\in\mathcal{N}(a)}\sum_{\tilde{a}'\in\mathcal{N}(a')}\mathcal{K}^{\tilde{a},\tilde{a}'}_{\ell-1}(x,x')\label{eq:nngp_loc_agg}\\
    \Theta^{a,a'}_\ell(x,x')&=\mathcal{K}^{a,a'}_{\ell}(x,x')+\sum_{\tilde{a}\in\mathcal{N}(a)}\sum_{\tilde{a}'\in\mathcal{N}(a')}\Theta^{\tilde{a},\tilde{a}'}_{\ell-1}(x,x')\label{eq:ntk_loc_agg}\,.
\end{align}
In a global aggregation layer, we sum over the entire graph instead. Correspondingly, for global aggregation layers, the sums in \eqref{eq:nngp_loc_agg} and \eqref{eq:ntk_loc_agg} run over the entire node set.

The kernels for the entire network are then given by the kernels of the last layer $L$,
\begin{align}
    \mathcal{K}^{a,a'}(x,x')=\mathcal{K}_L^{a,a'}(x,x')
    \qquad\text{and}\qquad
    \Theta^{a,a'}(x,x')=\Theta_{L}^{a,a'}(x,x')\,.
\end{align}
The presence or absence of spatial indices of the kernels depends on if the final network layer has spatial dimensions or not. If additional channels are present in the final layer (as e.g.\ in multi-class classification), the NTK is proportional to the unit matrix in those dimensions~\cite{jacot2018}. The fully general NTK therefore has the structure
\begin{align}
    \Theta^{\alpha a,\alpha' a'}(x,x')=\Theta^{a,a'}(x,x')\delta^{\alpha \alpha'}\,,\label{eq:index_structure_ntk}
\end{align}
where $\delta$ is the Kronecker symbol.

\textbf{Base case} Using the definition \eqref{eq:rhoX} of $\rho_X$, the NNGP is equivariant by orthogonality of $\tau_X$,
\begin{align}
    \mathcal{K}_1^{a,a'}(\rho_X(g)x,\rho_X(g)x')&=x^\top_{\rho(g)^{-1}a}\tau_X^\top(g)\tau_X(g)x'_{\rho(g)^{-1}a'}=x^\top_{\rho(g)^{-1}a}x'_{\rho(g)^{-1}a'}\\
    &=\mathcal{K}_1^{\rho(g)^{-1}a,\rho(g)^{-1}a'}(x,x')=(\rho_{K}(g)\mathcal{K}_{1}(x,x')\rho_{K}^\top(g))^{a,a'}\,.\label{eq:cnngp_1_trafo}
\end{align}
In the first layer, NNGP and NTK are equal according to \eqref{eq:cntk_1}, so also the NTK transforms as \eqref{eq:cnngp_1_trafo}.

\textbf{Induction step} Assume that
\begin{align}
    \mathcal{K}_{\ell-1}(\rho_X(g)x,\rho_X(g)x')&=\rho_K(g)\mathcal{K}_{\ell-1}(x,x')\rho_K^\top(g)\\
    \Theta_{\ell-1}(\rho_X(g)x,\rho_X(g)x')&=\rho_K(g)\Theta_{\ell-1}(x,x')\rho_K^\top(g)\,.
\end{align}
with $\rho_K$ trivial if no spatial indices are present in layer $\ell-1$. For the nonlinearities, we have
\begin{align}
    \mathcal{K}_{\ell}^{a,a'}(\rho_X(g)x,\rho_X(g)x')
    &=\mathbb{E}_{(u,v)\sim\mathcal{N}(0,\Lambda_{\ell-1}^{a,a'}(\rho_X(g)x,\rho_X(g)x'))}[\sigma(u)\sigma(v)]\\
    &=\mathbb{E}_{(u,v)\sim\mathcal{N}(0,\Lambda_{\ell-1}^{\rho^{-1}(g)a,\rho^{-1}(g)a'}(x,x'))}[\sigma(u)\sigma(v)]\\
    &=\mathcal{K}_{\ell}^{\rho^{-1}(g)a,\rho^{-1}(g)a'}(x,x')\\
    &=(\rho_K(g)\mathcal{K}_{\ell}(x,x')\rho_K^\top(g))^{a,a'}\label{eq:nngp_ind_step}
\end{align}
and similarly for $\dot{\mathcal{K}}_{\ell}$. For the NTK, we have
\begin{align}
    \Theta^{a,a'}_{\ell}(\rho_X(g)x,\rho_X(g)x')&=\dot{\mathcal{K}}_{\ell}^{a,a'}(\rho_X(g)x,\rho_X(g)x')\Theta^{a,a'}_{\ell-1}(\rho_X(g)x,\rho_X(g)x')\\
    &=(\rho_K(g)\dot{\mathcal{K}}_{\ell}(x,x')\rho_K^\top(g))^{a,a'}(\rho_K(g)\Theta_{\ell-1}(x,x')\rho_K^\top(g))^{a,a'}\\
    &=(\rho_K(g)\dot{\mathcal{K}}_{\ell}(x,x')\Theta_{\ell-1}(x,x')\rho_K^\top(g))^{a,a'}\\
    &=(\rho_K(g)\Theta^{a,a'}_{\ell}(x,x')\rho_K^\top(g))^{a,a'}\,.
\end{align}

For fully connected layers, the induction steps for $\mathcal{K}_\ell$ and $\Theta_\ell$ are implied immediately by \eqref{eq:sigma_forward_dense} and \eqref{eq:ntk_forward_dense} and the induction assumptions.

For convolutional layers, the induction step for $\mathcal{K}_{\ell}$ is given by
\begin{align}
    \mathcal{K}_\ell^{a,a'}(\rho_X(g)x,\rho_X(g)x')
    &=\sum_{\tilde{a}}\mathcal{K}_{\ell-1}^{a+\tilde{a},a'+\tilde{a}}(\rho_X(g)x,\rho_X(g)x')\\
    &=\sum_{\tilde{a}}\mathcal{K}_{\ell-1}^{\rho^{-1}(g)(a+\tilde{a}),\rho^{-1}(g)(a'+\tilde{a})}(x,x')\\
    &=\sum_{\tilde{a}}\mathcal{K}_{\ell-1}^{\rho^{-1}(g)a+\tilde{a},\rho^{-1}(g)a'+\tilde{a}}(x,x')\\
    &=\mathcal{K}_{\ell}^{\rho^{-1}(g)a,\rho^{-1}(g)a'}(x,x')\\
    &=(\rho_K(g)\mathcal{K}_\ell(x,x')\rho_K^\top(g))^{a,a'}\,.
\end{align}
The induction step for the NTK in convolutional layers proceeds along the same lines.

For flattening layers, the induction step for $\mathcal{K}_\ell$ is given by
\begin{align}
    \mathcal{K}_{\ell}(\rho_X(g)x,\rho_X(g)x')&=\sum_{\tilde{a}}\mathcal{K}_{\ell}^{\tilde{a},\tilde{a}}(\rho_X(g)x,\rho_X(g)x')\\ 
    &=\sum_{\tilde{a}}\mathcal{K}_{\ell}^{\rho^{-1}(g)\tilde{a},\rho^{-1}(g)\tilde{a}}(x,x')\\ 
    &=\sum_{\tilde{a}}\mathcal{K}_{\ell}^{\tilde{a},\tilde{a}}(x,x')=\mathcal{K}_\ell(x,x')\,. 
\end{align}
The induction step for the NTK of flattening layers proceeds along the same lines.

For local aggregation layers in graph neural networks, the induction step for $\mathcal{K}_{\ell}$ is given by
\begin{align}
    \mathcal{K}_\ell^{a,a'}(\rho_X(g)x,\rho_X(g)x')
    &=\sum_{\tilde{a}\in\mathcal{N}(a)}\sum_{\tilde{a}'\in\mathcal{N}(a')}\mathcal{K}^{\tilde{a},\tilde{a}'}_{\ell-1}(\rho_X(g)x,\rho_X(g)x')\\
    &=\sum_{\tilde{a}\in\mathcal{N}(a)}\sum_{\tilde{a}'\in\mathcal{N}(a')}\mathcal{K}^{\tilde{a},\tilde{a}'}_{\ell-1}(x,x')\\
    &=\mathcal{K}_\ell^{a,a'}(x,x')\\
    &=(\rho_K(g)\mathcal{K}_\ell(x,x')\rho_K^\top(g))^{a,a'}\,,
\end{align}
where we have used that $\rho=\mathbb{I}$ in this case. The induction step for the NTK of local aggregation layers proceeds along the same lines.

\end{proof}

\begin{corollary}\label{cor:rhoXshift}
    By redefining $x'\rightarrow\rho_X^{-1}(g)x'$, the transformation properties of NNGP and NTK in Theorem \ref{lemma1} can equivalently be written as
    \begin{align}
        \mathcal{K}(\rho_X(g)x,x')&=\rho_K(g)\mathcal{K}(x,\rho_X^{-1}(g)x')\rho_K^\top(g)\\
        \Theta(\rho_X(g)x,x')&=\rho_K(g)\Theta(x,\rho_X^{-1}(g)x')\rho_K^\top(g)\,.
    \end{align}
\end{corollary}

\begin{lemma}\label{app:helper_lemma}
Data augmentation implies that
\begin{enumerate}[label=(\alph*)]
    \item  $\Theta_{\pi_g(i), \, j} = \rho_K(g)\Theta^{}_{i, \, \pi_g^{-1}(j)}\rho_K^\top(g)$ \,,
    \item  $\Theta^{-1}_{\pi_g(i), \, j} = \rho_K(g)\Theta^{-1}_{i, \, \pi_g^{-1}(j)}\rho_K^\top(g)$ \,,
    \item $\mathcal{K}_{\pi_g(i), j} = \rho_K(g)\mathcal{K}_{i, \, \pi_g^{-1}(j)}\rho_K^\top(g)$ \,,
    \item $\mathcal{K}^{-1}_{\pi_g(i), \, j} = \rho_K(g)\mathcal{K}^{-1}_{i, \, \pi_g^{-1}(j)}\rho_K^\top(g)$ \,,
\end{enumerate}
and analogous results hold for any power of $\Theta$, $\Theta^{-1}$, $\mathcal{K}$ and $\mathcal{K}^{-1}$, respectively.
\end{lemma}
\begin{proof}
\textbf{(a):} By data augmentation, it follows that
\begin{align}
    \Theta_{\pi_g(i), \, j} &= \Theta(x_{\pi_g(i)}, x_j) \\
    &= \Theta(\rho_X(g) x_i, x_j) \\
    &= \rho_K(g)\Theta(x_i, \rho_X(g)^{-1} x_j)\rho_K^\top(g) \label{eq:NTKshiftcorusage}\\
    &= \rho_K(g)\Theta(x_i, x_{\pi_g^{-1}(j)})\rho_K^\top(g) \\
    &= \rho_K(g)\Theta_{i, \, \pi_g^{-1}(j)}\rho_K^\top(g) \,,
\end{align}
where we used Corollary~\ref{cor:rhoXshift} in \eqref{eq:NTKshiftcorusage}. For any power $N \in \mathbb{N}$ of the kernel, it holds therefore that
\begin{align}
    \left[\Theta^N\right]_{\pi_g(i), \, j} &= \Theta_{\pi_g(i), \, l} \left[ \Theta^{N-1} \right]_{l j}   \\
    &= \rho_K(g)\Theta_{i, \, \pi_g^{-1}(l)}\rho_K^\top(g) \left[ \Theta^{N-1} \right]_{l j} \\
    &\hspace{-2.25ex}\stackrel{l \mapsto \pi_g(l)}{=} \rho_K(g)\Theta_{i l}\rho_K^\top(g) \left[ \Theta^{N-1} \right]_{\pi_g(l) j} \\
    &= \rho_K(g)\Theta_{i l}\rho_K^\top(g) \Theta_{\pi(l)m}\left[ \Theta^{N-2} \right]_{m j} \\
    &= \rho_K(g)\Theta_{i l}\rho_K^\top(g) \rho_K(g)\Theta_{l\pi^{-1}_g(m)}\rho_K^\top(g)\left[ \Theta^{N-2} \right]_{m j} \label{eq:rhoKrohKT}\\
    &= \rho_K(g)\Theta^2_{i \pi^{-1}_g(m)}\rho_K^\top(g)\left[ \Theta^{N-2} \right]_{m j} \\
    &=\cdots=\rho_K(g)\left[\Theta^N\right]_{i, \, \pi^{-1}_g(j)}\rho_K^\top(g) \,.
\end{align}
Here, the contraction of spatial axes between adjacent kernels is implicit and in \eqref{eq:rhoKrohKT}, we have redefined these summation variables over spatial axes to absorb the action of $\rho_K^\top\rho_K$.

\textbf{(b):} We start from the equality
\begin{align}
    \Theta(\mathcal{X}, \rho_X(g) \mathcal{X})_{il} \left[ \Theta(\mathcal{X}, \rho_X(g) \mathcal{X}) \right]^{-1}_{lj} = \delta_{ij} \,,
\end{align}
where we have used the following notation for the Gram matrix $\Theta(\mathcal{X}, \mathcal{X})_{ij} := \Theta_{ij}$ and $G$ acts sample-wise on the dataset, $(\rho_X(g)\mathcal{X})_i=\rho_X(g)x_i$.
By data augmentation, this can be rewritten as
\begin{align}
    \Theta(\mathcal{X}, \mathcal{X})_{i, \,\pi_g(l)} \left[ \Theta(\mathcal{X}, \rho_X(g)\mathcal{X}) \right]^{-1}_{lj} = \delta_{ij} \,.
\end{align}
We now relabel the summation variable $l \to \pi_g^{-1}(l)$ and obtain
\begin{align}
    \Theta(\mathcal{X}, \mathcal{X})_{il} \left[ \Theta(\mathcal{X}, \rho_X(g)\mathcal{X}) \right]^{-1}_{\pi^{-1}_g(l), \,j} = \delta_{ij} \,.
\end{align}
By uniqueness of the inverse matrix, it thus follows that
\begin{align}
    \Theta(\mathcal{X}, \mathcal{X})^{-1}_{l j} = \left[ \Theta(\mathcal{X}, \rho_X(g) \mathcal{X}) \right]^{-1}_{\pi^{-1}_g(l), \,j} && \Longleftrightarrow &&
    \Theta(\mathcal{X}, \mathcal{X})^{-1}_{\pi_g(l),\, j} = \left[ \Theta(\mathcal{X}, \rho_X(g) \mathcal{X}) \right]^{-1}_{lj} \,. \label{eq:1sthelper}
\end{align}
Similarly, we can start from the expression
\begin{align}
    \left[ \Theta(\rho_X(g)^{-1}\mathcal{X}, \mathcal{X}) \right]^{-1}_{il} \Theta(\rho_X^{-1}(g) \mathcal{X}, \mathcal{X})_{l j} = \delta_{ij} \,.
\end{align}
By data augmentation, this can be rewritten as
\begin{align}
    \left[ \Theta(\rho_X(g)^{-1}\mathcal{X}, \mathcal{X}) \right]^{-1}_{il} \Theta(\mathcal{X}, \mathcal{X})_{\pi_g^{-1}(l),\, j} = \delta_{ij} \,.
\end{align}
Relabeling the summation variable $l \to \pi_g(l)$, we obtain
\begin{align}
    \left[ \Theta(\rho_X^{-1}(g) \mathcal{X}, \mathcal{X}) \right]^{-1}_{i, \pi_g(l)} \Theta(\mathcal{X}, \mathcal{X})_{l j} = \delta_{ij} \,.
\end{align}
By uniqueness of the inverse matrix, it follows again that
\begin{align}
    \Theta(\mathcal{X}, \mathcal{X})^{-1}_{i l} = \left[ \Theta(\rho_X^{-1}(g) \mathcal{X}, \mathcal{X}) \right]^{-1}_{i, \pi_g(l)} && \Longleftrightarrow && \Theta(\mathcal{X}, \mathcal{X})^{-1}_{i, \pi^{-1}_g(l)} = \left[ \Theta(\rho_X^{-1}(g)\mathcal{X}, \mathcal{X}) \right]^{-1}_{i l} \label{eq:2ndhelper}  \,.
\end{align}
Combining the results \eqref{eq:1sthelper} and \eqref{eq:2ndhelper}, the statement of the lemma follows immediately:
\begin{align}
    \Theta^{-1}_{\pi_g(i), j} = \left[\Theta(\mathcal{X}, \mathcal{X})\right]^{-1}_{\pi_g(i), j} &\hspace{-0.75ex}\stackrel{\eqref{eq:1sthelper}}{=} \left[\Theta(\mathcal{X}, \rho_X(g)\mathcal{X})\right]^{-1}_{i j} \\
    &= \rho_K(g)\left[\Theta(\rho_X^{-1}(g) \mathcal{X}, \mathcal{X})\right]^{-1}_{i j}\rho_K^\top(g) \\
    &\hspace{-0.75ex}\stackrel{\eqref{eq:2ndhelper}}{=} \rho_K(g)\Theta(\mathcal{X}, \mathcal{X})^{-1}_{i, \pi^{-1}_g(j)}\rho_K^\top(g) \\
    &= \rho_K(g)\Theta^{-1}_{i, \pi^{-1}_g(j)}\rho_K^\top(g) \,,
\end{align}
where $\rho_K$ is unaffected by the inverse since we invert the Gram matrix along the training sample axes $i,j$. The proof for any power $(\Theta^{-1})^N$ of the inverse Gram matrix follows in complete analogy to the proof of the same result for the Gram matrix $\Theta$.

\textbf{(c):} The proof for the NNGP follows in close analogy to the one for the NTK, see (a):
\begin{align}
    \mathcal{K}_{\pi_g(i), j} &= \mathcal{K}(x_{\pi_g(i)}, x_j) = \mathcal{K}(\rho_X(g) x_i, x_j)\\
    &= \rho_K(g)\mathcal{K}(x_i, \rho_X^{-1}(g) x_j)\rho_K^\top(g) = \rho_K(g)\mathcal{K}(x_i, x_{\pi^{-1}_g(j)})\rho_K^\top(g) \\
    &= \rho_K(g)\mathcal{K}_{i, \, \pi_g^{-1}(j)}\rho_K^\top(g) \,.
\end{align}
The proof for any power of the NNGP again follows in complete analogy to (a).

\textbf{(d):} Since the transformation properties of $\Theta$ and $\mathcal{K}$ under $G$ are completely identical, the proof follows the steps of (b) verbatim with the replacement $\Theta\rightarrow\mathcal{K}$. Similarly for any power of $\mathcal{K}$.
\end{proof}

Using this result, we can then show the following lemma as stated in the main part:
\secondlemma*
\begin{proof}
    As the matrix-valued function $F$ is analytic, it has the following series expansion
    \begin{align}
        F(\Theta, \Theta^{-1}, \mathcal{K}, \mathcal{K}^{-1})_{ij} = \sum_{n=1}^\infty \sum_{P_n} c_{P_n} \; P_n(\Theta, \Theta^{-1}, \mathcal{K}, \mathcal{K}^{-1})_{ij} \,,
    \end{align}
    where the inner sum is over all order $n$ polynomials involving $\Theta$ and $\mathcal{K}$ as well as their inverses and $c_{P_n}$ are coefficients.

By Lemma~\ref{app:helper_lemma}, for any such polynomial $P_n$ we have
\begin{align}
    P_n(\Theta, \Theta^{-1}, \mathcal{K}, \mathcal{K}^{-1})_{\pi_g(i) j} = \rho_K(g)P_n(\Theta, \Theta^{-1}, \mathcal{K}, \mathcal{K}^{-1})_{i \pi^{-1}_g(j)}\rho_K^\top(g) \,.
\end{align}
Applying this result to the series expansion above implies
\begin{align}
    \big[\Pi(g)F(\Theta, \Theta^{-1}, \mathcal{K}, \mathcal{K}^{-1})\big]_{ij}&=F(\Theta, \Theta^{-1}, \mathcal{K}, \mathcal{K}^{-1})_{\pi_g(i)j}\\
    &=\rho_K(g)F(\Theta, \Theta^{-1}, \mathcal{K}, \mathcal{K}^{-1})_{i\pi_g^{-1}(j)}\rho_K^\top(g)\\
    &=\rho_K(g)\big[F(\Theta, \Theta^{-1}, \mathcal{K}, \mathcal{K}^{-1})(\Pi^{-1}(g))^\top\big]_{ij}\rho_K^\top(g)\\
    &=\rho_K(g)\big[F(\Theta, \Theta^{-1}, \mathcal{K}, \mathcal{K}^{-1})\Pi(g)\big]_{ij}\rho_K^\top(g)\,.
\end{align}
\end{proof}

\thmemergingequiv*
\begin{proof}
The mean function of the output distribution on a test sample $x$ after training time $t$ is according to \eqref{eq:output_mean} given by 
\begin{align}
    \mu(\rho_X(g)x)
    &=\Theta(\rho_X(g)x,\mathcal{X})[\Theta^{-1}T_t]\mathcal{Y})\\
    &=\rho_K(g)\Theta(x,\rho^{-1}_X(g)\mathcal{X})\rho_K^\top(g)[\Theta^{-1}T_t]\mathcal{Y}\\
    &=\rho_K(g)\Theta(x,\mathcal{X})\rho_K^\top(g)\Pi(g)[\Theta^{-1}T_t]\mathcal{Y}\\
    &=\rho_K(g)\Theta(x,\mathcal{X})[\Theta^{-1}T_t]\rho_K^\top(g)\Pi(g)\mathcal{Y}\\
    &=\rho_K(g)\Theta(x,\mathcal{X})[\Theta^{-1}T_t]\rho_K^\top(g)\rho_Y(g)\mathcal{Y}\,,
\end{align}
On a label with spatial index $a$ and channel index $\alpha$, $\rho_Y$ acts according to \eqref{eq:rhoY}. Furthermore, the index structure of the NTK will match the index structure of the labels since we use the network outputs to predict the labels. If the labels carry a channel index, then $\Theta$ is proportional to the unit matrix in this index, as mentioned in \eqref{eq:index_structure_ntk} and hence the representation $\tau_Y$ commutes all the way to the left. Finally, the action of $\rho$ on the spatial indices of the labels (if present) is the same as the action of $\rho_K$, so we obtain
\begin{align}
    \mu(\rho_X(g)x)&=\tau_Y(g)\rho_K(g)\Theta(x,\mathcal{X})[\Theta^{-1}T_t]\rho_K^\top(g)\rho_K(g)\mathcal{Y}\\
    &=\tau_Y(g)\rho_K(g)\Theta(x,\mathcal{X})[\Theta^{-1}T_t]\mathcal{Y}\\
    &=\tau_Y(g)\rho_K(g)\mu(x)\\
    &=\rho_Y(g)\mu(x)\,.\label{eq:murhok_trafo}
\end{align}
The covariance function transforms according to
\begin{align}
    \Sigma_t(\rho_X(g)x,\rho_X(g)x') &= \mathcal{K}(\rho_X(g)x, \rho_X(g)x') + \Sigma_t^{(1)}(\rho_X(g)x,\rho_X(g)x') - ( \Sigma_t^{(2)}(\rho_X(g)x,\rho_X(g)x') + \textrm{h.c.} ) \,.
\end{align}
The transformation of $\mathcal{K}$ is given by Theorem \ref{lemma1}. The transformation of $\Sigma_t^{(1)}$ is given by
\begin{align}
    &\Sigma_t^{(1)}(\rho_X(g) x,\rho_X(g)x')\nonumber\\
    =&\ \Theta(\rho_X(g) x, \mathcal{X})\,(\Theta(\mathcal{X},\mathcal{X}))^{-1}\,T_t\, \mathcal{K}(\mathcal{X},\mathcal{X})\,T_t\, (\Theta(\mathcal{X},\mathcal{X}))^{-1}\,\Theta(\mathcal{X}, \rho_X(g) x') \\
    =&\ \rho_K(g)\Theta(x, \mathcal{X})\rho_K^\top(g)\Pi(g)\,(\Theta(\mathcal{X},\mathcal{X}))^{-1}\,T_t\, \mathcal{K}(\mathcal{X},\mathcal{X})\,T_t\, (\Theta(\mathcal{X},\mathcal{X}))^{-1}\,\Pi^\top(g)\rho_K(g)\Theta(\mathcal{X}, x')\rho_K^\top(g) \\
    =&\ \rho_K(g)\Theta(x, \mathcal{X})\,(\Theta(\mathcal{X},\mathcal{X}))^{-1}\,T_t\, \mathcal{K}(\mathcal{X},\mathcal{X})\,T_t\, (\Theta(\mathcal{X},\mathcal{X}))^{-1}\,\rho_K^\top(g)\Pi(g)\Pi^\top(g)\rho_K(g)\Theta(\mathcal{X}, x')\rho_K^\top(g) \\
    =&\ \rho_K(g)\Theta(x, \mathcal{X})\,(\Theta(\mathcal{X},\mathcal{X}))^{-1}\,T_t\, \mathcal{K}(\mathcal{X},\mathcal{X})\,T_t\, (\Theta(\mathcal{X},\mathcal{X}))^{-1}\,\Theta(\mathcal{X}, x')\rho_K^\top(g) \\
    =&\ \rho_K(g)\Sigma_t^{(1)}(x,x')\rho_K^\top(g)\,.
\end{align}
Similarly for $\Sigma_t^{(2)}$. In total, the covariance function transforms according to
\begin{align}
    \Sigma_t(\rho_X(g)x,\rho_X(g)x')=\rho_K(g)\Sigma_t(x,x')\rho_K^\top(g)\,.
\end{align}
Since the covariance function is also proportional to the unit matrix in possible channel dimensions, adding a transformation of $\tau_Y$ on the left and $\tau_Y^\top$ on the right does not change the expression. Therefore
\begin{align}
    \Sigma_t(\rho_X(g)x,\rho_X(g)x')=\rho_Y(g)\Sigma_t(x,x')\rho_Y^\top(g)\,.
\end{align}
Together with \eqref{eq:murhok_trafo}, this implies that the output distribution is equivariant w.r.t.\ $\rho_Y$ for any training time $t$ and for any input $x$.
\end{proof}

\subsection{Finite Number of Ensemble Members}

\begin{lemma}
    The probability that the deep ensemble $\bar{f}_t$ and its estimate $\hat{f}_t$ differ by more than a given threshold $\delta$ is bounded by
    \begin{align}
        \mathbb{P} \left[ |\hat{f}_t(x) - \bar{f}_t(x)| > \delta \right] \le \sqrt{\frac{2}{\pi}} \, \frac{\sigma_x}{\delta} \exp\left( - \frac{\delta^2}{2 \sigma_x^2} \right) \,,\label{eq:sample_prob_bound}
    \end{align}
    where we have defined
    \begin{align}
        \sigma_x^2 := \textrm{Var}(\hat{f}_t)(x) = \frac{\Sigma_t(x)}{M}
    \end{align}
    with the output variance $\Sigma_t(x)=\Sigma_t(x,x)$ defined in \eqref{eq:output_var}.
\end{lemma}
\begin{proof}
   The probability of such deviations is given by
    \begin{align}
       \mathbb{P} \left[ |\hat{f}_t(x) - \bar{f}_t(x)| 
 > \delta \right] = \frac{2}{\sqrt{2 \pi} \sigma_x} \int_{\delta}^{\infty} \exp\left(-\frac{t^2}{2\sigma^2_x} \right)  \mathrm{d}t 
    \end{align}
    We now change the integration variable to $\tau = \frac{t}{\sigma_x \sqrt{2}}$ and obtain
    \begin{align}
        \mathbb{P} \left[ |\hat{f}_t(x) - \bar{f}(_tx)| > \delta \right] = \frac{2}{\sqrt{\pi}} \int_{\frac{\delta}{\sqrt{2} \sigma_x}}^\infty \exp(-\tau^2) \mathrm{d}\tau \le \frac{1}{\sqrt{\pi}} \, \frac{\sqrt{2} \sigma_x}{\delta} \, \int_{\frac{\delta}{\sqrt{2} \sigma_x}}^\infty (2\tau) \, \exp(-\tau^2) \mathrm{d}\tau \,,
    \end{align}
    where we have used that $1 \le \frac{2\tau}{2\min(\tau)}$ for $\tau \ge \min(\tau)$ to obtain the last inequality. The integral can be straightforwardly evaluated by rewriting the integrand as a total derivative and we thus obtain 
    \begin{align}
        \mathbb{P} \left[ |\hat{f}_t(x) - \bar{f}_t(x)| > \delta \right] \le \sqrt{\frac{2}{\pi}} \, \frac{\sigma_x}{\delta} \exp\left( - \frac{\delta^2}{2 \sigma^2_x} \right) \,.
    \end{align}
\end{proof}
We stress that this result holds for any Monte-Carlo estimator and we therefore suspect that it could be well-known. For most MC estimators, it is however of relatively little use as the variance $\Sigma$ is not known in closed form --- in stark contrast to the deep ensemble, see \eqref{eq:output_var}, considered in this paper. This could explain why we were not able to locate this result in the literature.

For the deep ensemble, we can therefore exactly determine the necessary number of ensemble size to stay within a certain threshold $\delta$ with a given probability $1-\epsilon$. For this, one has to set the right-hand-side of the derived expression to this confidence $\epsilon$ and solve for the necessary ensemble size $M$.
However, this equation appears to have no closed-from solution and needs to be solved numerically. We advise the reader to do so if need for a tight bound arises. For the presentation in the main part, we however wanted to derive a closed-form solution for $M$ and thus had to rely on a looser bound which implies the following statement:

\lemboundfiniteensemble*
\begin{proof}
\begin{align}
    \mathbb{P} \left[ |\hat{f}_t(x) - \bar{f}_t(x)| > \delta \right] < \frac{1}{\sqrt{\pi}} \, \frac{1}{z} \exp\left( - z^2 \right) \le \frac{1}{\sqrt{\pi}} \,  \exp\left( - z^2 \right) \overset{!}{<} \epsilon
\end{align} 
with $z = \frac{\delta}{\sqrt{2} \sigma_x}$ and where we assume that $M$ is chosen sufficiently large such that $z \ge 1$. This implies that
\begin{align}
    z^2 > - \ln (\sqrt{\pi} \epsilon) \qquad \Leftrightarrow \qquad M >  - \frac{2 \Sigma_t(x)}{\delta^2} \ln (\sqrt{\pi}  \epsilon) \,.
\end{align}
    
\end{proof}

\subsection{Continuous Groups}
\thboundcont*
\begin{proof}
As described in the main text, we consider a finite subgroup $A \subset G$ which we use for data augmentation (instead of using the continuous group $G$). The discretization error for the representation $\rho_X$ is given by
\begin{align}
    \epsilon = \max_{g \in G} \, \min_{g'\in A} \|\rho_X(g) - \rho_X(g')\|_{\mathrm{op}} \,.
\end{align}
This implies that for any $g \in G$, we can find a $g' \in A$ such that
\begin{align}
    \|\rho_X(g) x_i - x_{\pi_{g'}(i)}\| = \|\rho_X(g) x_i - \rho_X(g') x_{i}\| \le \|\rho_X(g) - \rho_X(g')\|_\mathrm{op} \, \|x_i\| < \epsilon \|x_i\| \,,
\end{align}
where we have used data augmentation \eqref{eq:perm_group_action} over $A$.

We can then calculate the difference of the prediction at any test point $x$ and its transformation:
\begin{align}
    |\bar{f}_t(x) - \bar{f}_t(\rho_X(g) x)|&=|\mu_t(x) - \mu_t(\rho_X(g) x)|\\
    &= |(\Theta(x, x_i) -  \Theta(\rho_X(g) x, x_i)) \, \Theta^{-1}_{ij} \, (\mathbb{I} - \exp(- \eta \Theta t) )_{jk} \,y_k |
\end{align}
From the Lemma~\ref{lemma2}, it follows that
\begin{align}
    \Theta(x, x_i) \, \Theta^{-1}_{ij} \, (\mathbb{I} - \exp(- \eta \Theta t) )_{jk} \,y_k &= \Theta(x, x_i) \, \Theta^{-1}_{ij} \, (\mathbb{I} - \exp(- \eta \Theta t) )_{jk} \,y_{\pi_{g'}(k)} \\
    &= \Theta(x, x_{\pi_{g'}^{-1}(i)} ) \, \Theta^{-1}_{ij} \, (\mathbb{I} - \exp(- \eta \Theta t) )_{jk} \,y_{k}
\end{align}
Thus the difference can be rewritten as follows
\begin{align}
    |\bar{f}_t(x) - \bar{f}_t(\rho_X(g) x)| &= | (\Theta(x, x_{\pi_{g'}^{-1}(i)}) -  \Theta(\rho_X(g) x, x_i)) \, \Theta^{-1}_{ij} \, (\mathbb{I} - \exp(- \eta \Theta t) )_{jk} \,y_k |  \\
    &=| (\Theta(x, x_{\pi_{g'}^{-1}(i)}) -  \Theta(x, \rho_X^{-1}(g)x_i)) \, \Theta^{-1}_{ij} \, (\mathbb{I} - \exp(- \eta \Theta t) )_{jk} \,y_k |  \label{eq:intermit_continous}
\end{align}
It is convenient to define
\begin{align}
 \Delta \Theta(x', x, \bar{x}) \equiv   | \Theta(x', x) -  \Theta(x', \bar{x}) |
\end{align}
which can be bounded as follows
\begin{align}
\Delta \Theta(x', x, \bar{x}) &= \left| \sum_{l=1}^L \mathbb{E}_{w \sim p} \left[  \left(\frac{\partial f_w(x')}{\partial w^{(l)}}\right)^\top \, \left( \frac{\partial f_w(x)}{\partial w^{(l)}} - \frac{\partial f_w(\bar{x})}{\partial w^{(l)}} \right) \right] \right| \\
&\le \|x - \bar{x}\| \, \sum_{l=1}^L \mathbb{E}_{w \sim p} \left[ \left|  \left(\frac{\partial f_w(x')}{\partial w^{(l)}}\right)^\top \cdot L(w^{(l)}) \right| \right] \\
&\equiv \|x-\bar{x}\| \, \hat{C}(x) \,,
\end{align}
where $L(w^{(l)})$ is the Lipschitz constant of $\partial_{w^{(l)}} f_w$ and  we emphasize that the norm is with respect to the input space. Using this expression, we can bound the difference of the means \eqref{eq:intermit_continous} by using the triangle inequality
\begin{align*}
    |\bar{f}_t(x) - \bar{f}_t(\rho_X(g) x)| &\le \hat{C}(x) \sqrt{\sum_i \|x_{\pi_{g'}^{-1}(i)} - \rho_X(g)^{-1} x_i\|^2} \, \sqrt{\sum_i (\sum_{j,k } \Theta^{-1}_{ij} \, (\mathbb{I} - \exp(- \eta \Theta t) )_{jk} \,y_k ] )^2  } \\
    &\le \epsilon \, \hat{C}(x) \sqrt{\sum_i \|x_i\|^2} \, \sqrt{\sum_i (\sum_{j,k } \Theta^{-1}_{ij} \, (\mathbb{I} - \exp(- \eta \Theta t) )_{jk} \,y_k ] )^2  } 
    \equiv \epsilon C(x) \,.
\end{align*}
Note that this result suggests that one should choose the discretization carefully to achieve as tight of a bound as possible.
\end{proof}

\section{Experiments}
In this section, we provide further details about our experiments.

\subsection{Ising Model}
\label{app:ising_model}

\begin{figure}[tb]
  \centering
  \includegraphics[width=\textwidth]{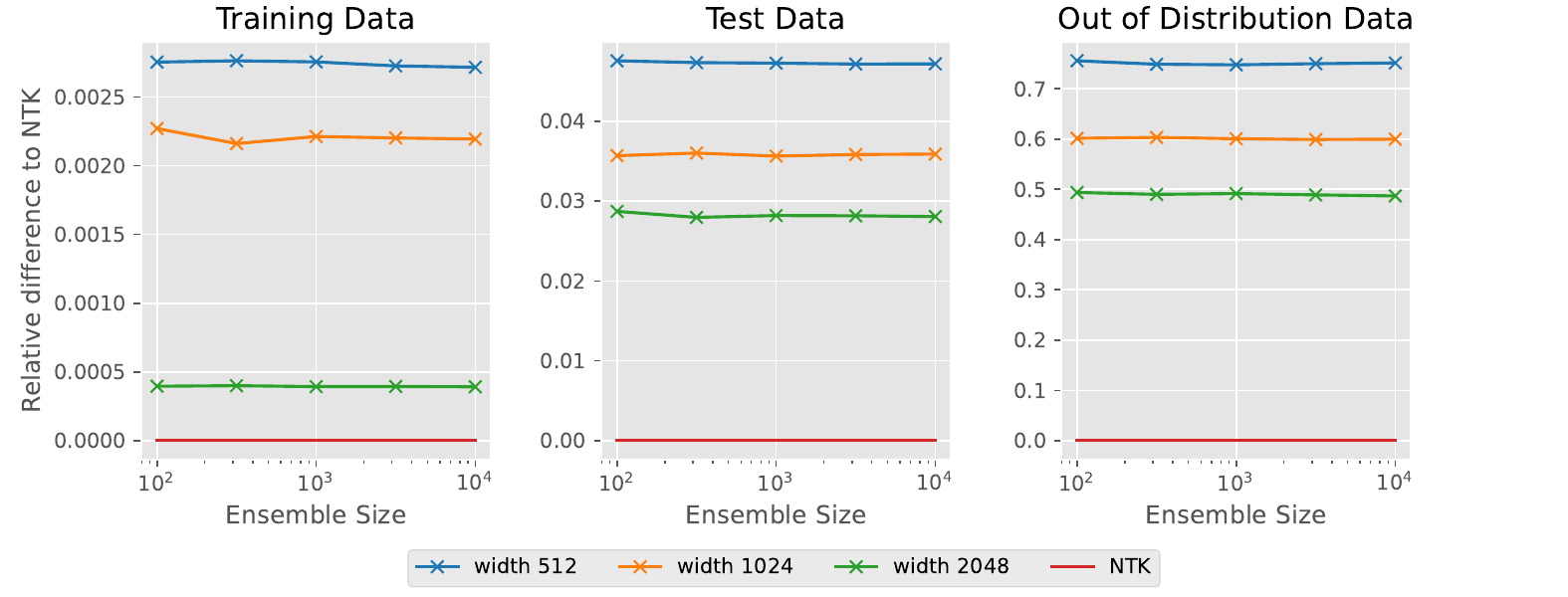}
  \caption{Difference in relative predicted total energy $\mathcal{E}$ between the ensembles and the NTK on the training data, in-distribution test data and out of distribution.}
  \label{fig:ntk_conv}
\end{figure}

\paragraph{Training details}
The energy function of the Ising model can be written as
\begin{align}
  \mathcal{E} = -\frac{J}{\mathrm{vol}(L)}\sum_{i\in L}E(i)\,,\label{eq:1}
\end{align}
where $J$ is a coupling constant which we set to one for convenience and $\mathrm{vol}(L)$ denotes the number of lattice sites. The local energy $E(i)$ is given by\footnote{Usually, one only sums over pairs of spins. Our prescription differs from that convention by an irrelevant factor of two and makes the local energy exactly equivariant under rotations of the lattice by $90^{\circ}$.}
\begin{align}
  E(i)=\sum_{j\in\mathcal{N}(i)}s_{i}s_{j}\,,\label{eq:2}
\end{align}
where $\mathcal{N}(i)$ denotes the neighbors of $i$ along the lattice axes. The expectation value of $\mathcal{E}$ vanishes and its standard deviation is $2$ for uniform sampling of spins in $\{+1,-1\}$.

The energy of the Ising model is invariant under rotations of the lattice by $90^{\circ}$, since the local energy \eqref{eq:2} stays invariant if the neighborhood is rotated and the sum in \eqref{eq:1} is just reshuffled. We train a fully-connected network with one hidden layer and a ReLU activation on $128$ samples augmented with full $C_{4}$ orbits to $512$ training samples. To obtain a sufficient training signal, we train the networks with a squared error loss on the local energies~\eqref{eq:2}. We train for 100k steps of full-batch gradient descent with learning rate $0.5$ for network widths $128$, $512$ and $1024$ and learning rate $1.0$ for network width $2048$.

\paragraph{Ensemble-convergence to the NTK}
We verify that the ensembles converge to the NTK for large widths by computing the difference in total energy $\mathcal{E}$ between the mean ensemble prediction and the predicted mean of the NTK, cf.\ Figure~\ref{fig:ntk_conv}. To make the numbers easily interpretable, we plot the relative difference, where we divide by the standard deviation of the ground truth energy, $2$, which gives a typical value for $\mathcal{E}$. We perform the comparisons on the training data, in-distribution test data and out of distribution data. As expected, agreement is highest on the training data and lowest out of distribution, but in each case, ensembles with higher-width hidden layer generate mean predictions closer to the NTK. Beyond ensemble size $1000$, the estimate of the expectation value over initializations in the NTK seems to be accurate enough that no further fluctuations can be seen in the plots.

\subsection{Rotated FashionMNIST}
\label{app:fmnist}

\paragraph{Ensemble architecture} As ensemble members, we use a simple convolutional neural network with two convolutional layers of kernel size $3$ and $6$ as well as $16$ channels respectively. Both convolutional layers are followed by a relu non-linearity as well as $2\times2$ max-pooling. This is then followed by layers fully-connected of size $(400,120)$, $(120, 84)$, and $(84, 10)$ of which the first two are fed into relu non-linearities. We choose ensembles of size $M=5, 10, 100$. 

\paragraph{OOD data} We use the validation set of greyscaled and rescaled CIFAR10, the validation set of MNIST, as well as a dataset generated by images with pixels drawn iid from $N(0,1)$ as OOD data. We also evaluate the invariance on the validation set of FMNIST, i.e., on in-distribution data. Please refer to the corresponding Figure~\ref{fig:app_fminst_rand}, \ref{fig:app_fminst_fminst}, and \ref{fig:app_fminst_cifar10} contained in this appendix for the results.

\paragraph{Data augmentation} We augment the original dataset by all elements of the group orbit of the cyclic group $C_k$, i.e., all rotations of the image by any multiple of $360/k$ degrees and ensure that each epoch contains all element of the group orbit in each epoch to closely align the experiments with our theoretical analysis. However, in exploratory analysis, we did not observe a notable difference when applying random group elements in each training step. For the cyclic group $C_k$, we choose group orders $k=4, 8, 16$.

\paragraph{Training details} We use the ADAM optimizer with the standard learning rate of pytorch lightning, i.e., 1e-3. We train for 10 epochs on the augmented dataset. We evaluate the metrics after each epoch on both the in-distribution and the out-of-distribution data. The ensembles achieve a test accuracy on the augmented datasets of between 88 to 91 percent depending on the chosen group order and ensemble size.

\paragraph{OSP metric:} To obtain the orbit same prediction, we measure
\begin{equation}
    \sum_{g \in G} \mathbb{I}(\textrm{argmax}_\alpha f^\alpha(\rho_X(g) x), \textrm{argmax}_\alpha f^\alpha(x) ) \,,
\end{equation}
where $\mathbb{I}$ denotes the indicator function. This corresponds to the number of elements in the orbit that have the same predicted class as the transformed data sample $x$. The orbit same prediction (OSP) of a dataset $\mathcal{D}$ is then this number averaged over all elements in the dataset. Note that the OSP has minimal value 1 as the identity is always part of the orbit.

\paragraph{Continuous rotations:} We analyze the generalization properties to the full two-dimensional rotation group $SO(2)$ for deep ensembles trained with data augmentation using the finite cyclic group $C_k$. To this end, we define the continuous orbit same prediction as:
\begin{align}
\frac{1}{\textrm{Vol}(SO(2))} \int_{SO(2)} \mathrm{d}g \, \mathbb{I}(\textrm{argmax}_\alpha f^\alpha(\rho_X(g) x), \textrm{argmax}_\alpha f^\alpha(x) ) \,,
\end{align}
where $\mathrm{d}g$ denotes the Haar measure. This continuous orbit same prediction thus corresponds to the percentage of elements in the orbit that are classified the same way as the untransformed element. We estimate this quantity by Monte-Carlo. The results of our analysis are shown in Figure~\ref{fig:so2_appendix} and clearly establish that for sufficiently high group order of the cyclic group used for data augmentation, the ensemble is approximately invariant with respect to the continuous symmetry as well. In particular, it is signficantly more invariant as its ensemble members. Interestingly, this is competitive with a model that is using canonicalization \cite{kaba2023equivariance} with respect to $C_k$ and the same network architecture as its predictor network.

\paragraph{Comparison to other methods:} 
For the deep ensemble, we use ten ensemble members with the same convolututional architecture as outlined above. For canonicalization, we use the same convolutional architecture as for the ensemble members and the same architecture for the canonicalization network as in the original publication \cite{kaba2023equivariance}. For E2CNN, we follow the official MNIST example. We adjust hyperparameters such that all methods use roughly the same number of parameters as the deep ensemble. As a result, all methods have roughly the same number of parameters. Figure~\ref{fig:app_method_comparison} demonstrates that all methods lead to a comparable degree of equivariance. Note that interpolation effects seem to hurt the performance of canonicalization more dramatically as compared to E2CNN. This is to be expected as canonicalization works by predicting a rotation and then undoing the rotation. This leads to another compounded source of discretization errors. 

\begin{figure}[tb]
  \centering
  \includegraphics[width=0.33\linewidth]{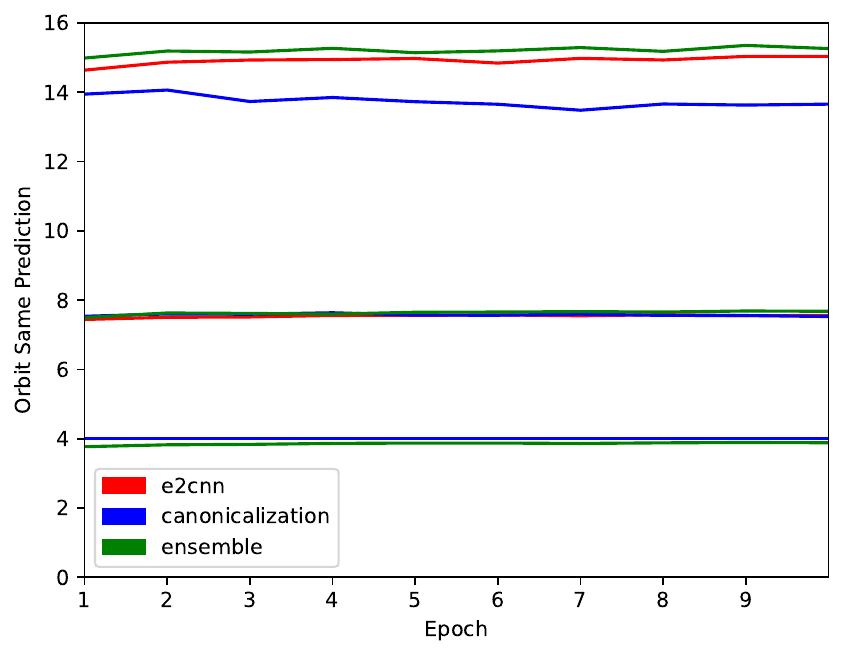}
  \hfill
  \includegraphics[width=0.33\linewidth]{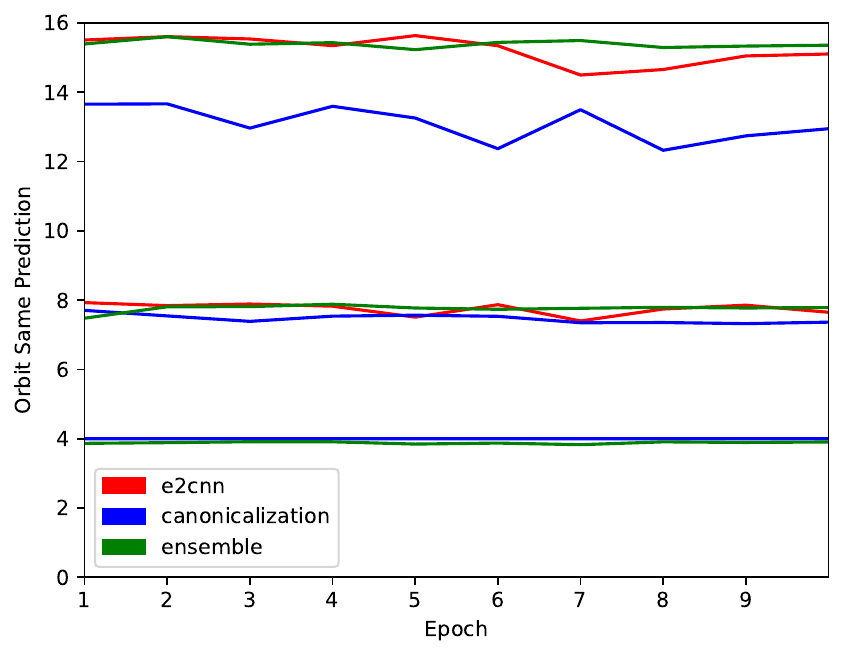}
  \hfill
  \includegraphics[width=0.33\linewidth]{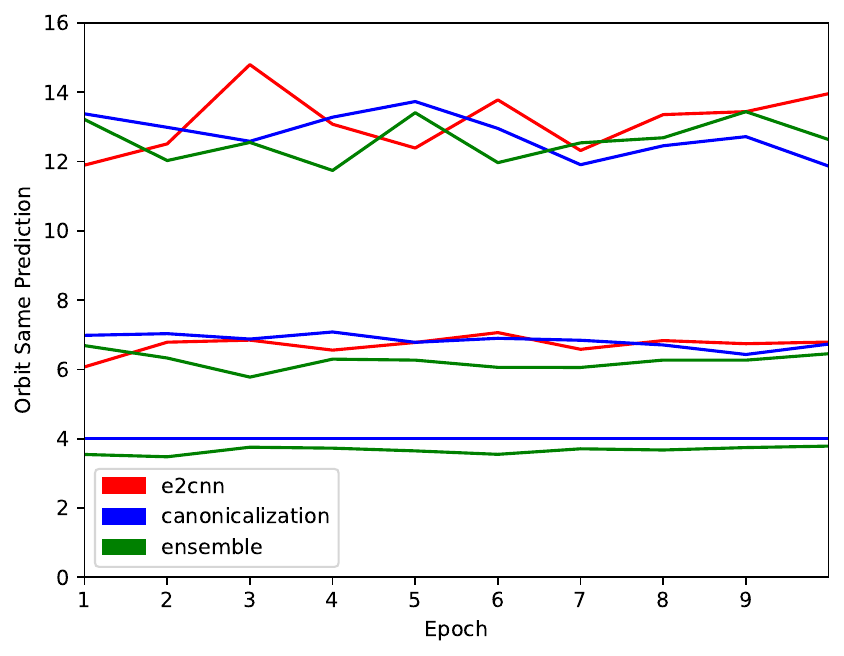}
 \caption{Comparison of various equivariance methods on in-distribution FMNIST (left), out-of-distribution MNIST (middle), out-of-distribution CIFAR10 (right). Deep ensembles are approximately equivariant due to finite-size ensembles and finite width (see discussion in main text). Canonicalization and E2CNN also do not show perfect equivariance for group orders $k>4$ because of interpolation artifacts, see, for example, discussion in \cite{kaba2023equivariance}.}
        \label{fig:app_method_comparison}
\end{figure}

\begin{figure}[t!]
  \centering
  \includegraphics[width=0.5\textwidth]{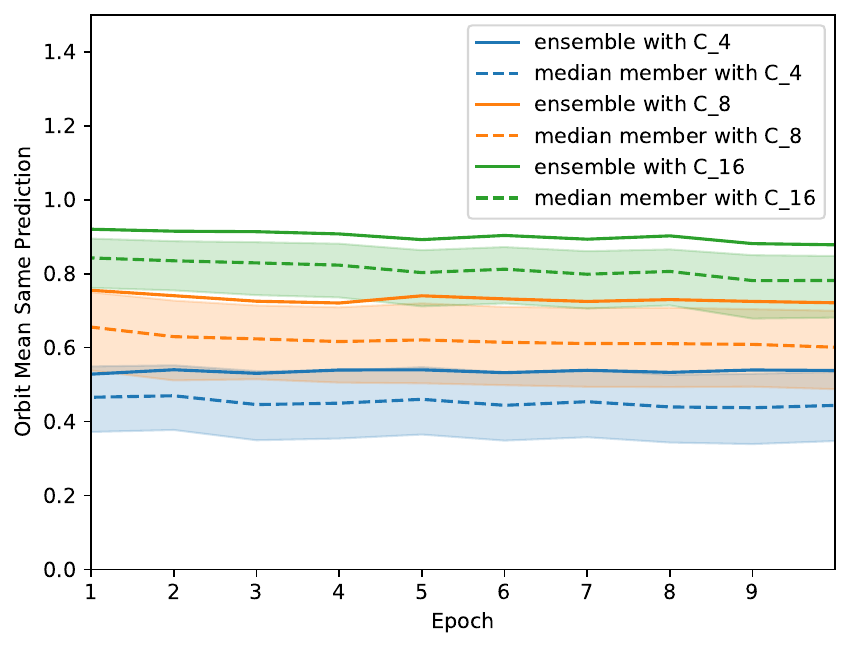}
  \caption{Mean orbit same prediction over $SO(2)$ group orbits. Solid lines show the ensemble prediction while dotted lines show the median of the ensemble members. Error band denotes the 75th and 25th percentile. As the group order $k$ of the cyclic group $C_k$ used for data augmentation increases, the mean orbit same prediction over $SO(2)$ increases. For $k=16$, over 90 percent of the orbit elements have the same prediction as the untransformed input establishing that the model is approximately invariant under the continuous symmetry as well. The invariance of the ensemble is again emergent in the sense that it is above the 75th percentile of the ensemble members.}
  \label{fig:so2_appendix}
\end{figure}

\begin{figure}[t!]
  \centering
  \includegraphics[width=0.5\textwidth]{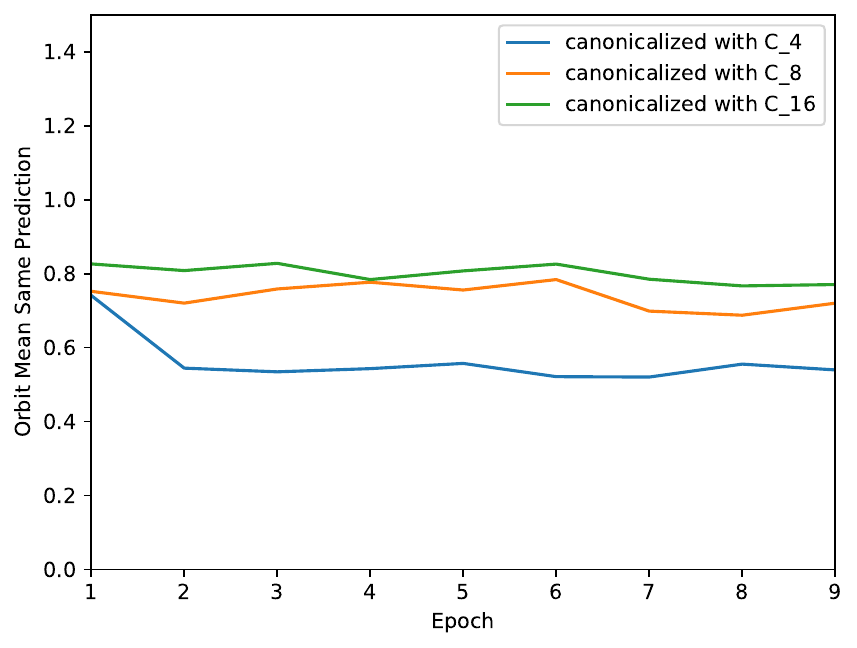}
  \caption{Mean orbit same prediction over $SO(2)$ group orbits for a model canonicalized with respect to $C_k$. As the group order $k$ of the cyclic group $C_k$ used for data augmentation increases, the mean orbit same prediction over $SO(2)$ increases.}
  \label{fig:so2_canonicalization}
\end{figure}

\begin{figure}[tb]
  \centering
  \includegraphics[width=0.48\linewidth]{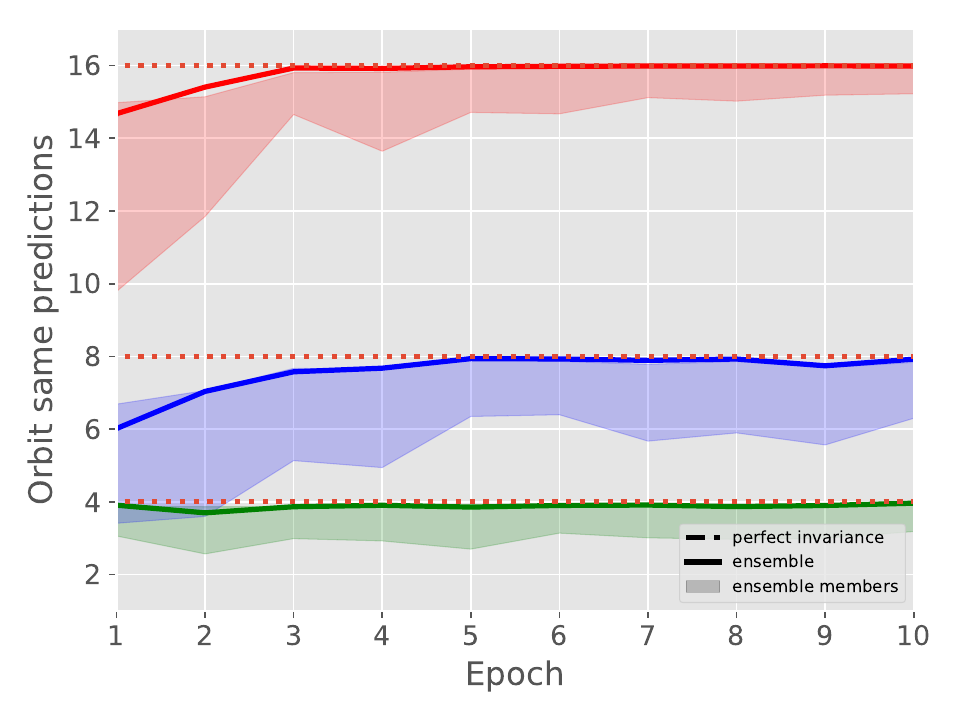}
  \hfill
  \includegraphics[width=0.48\linewidth]{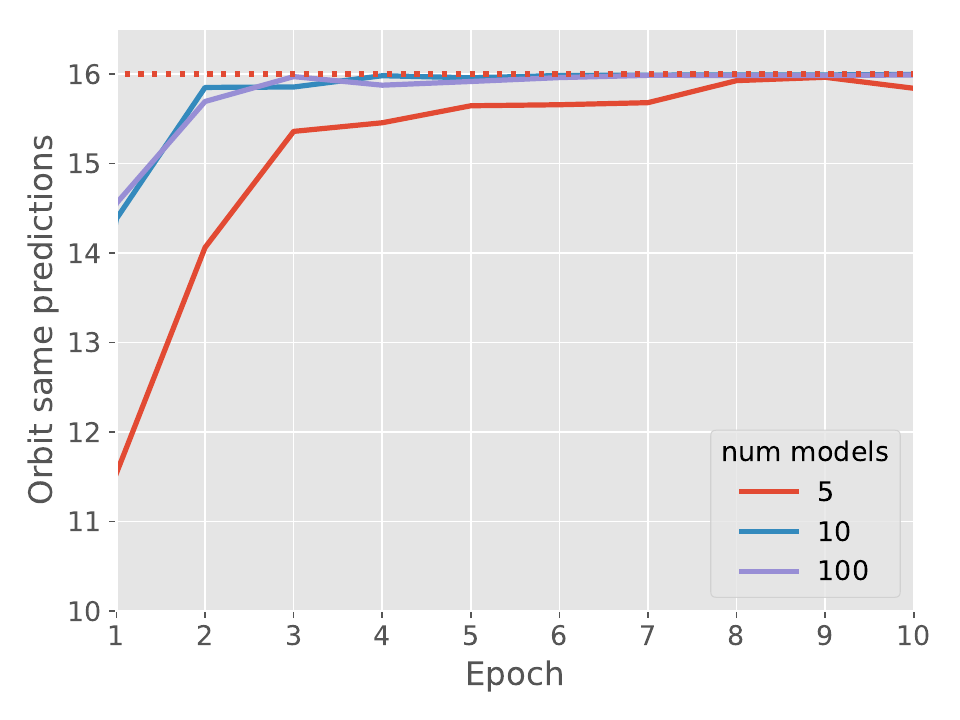}
 \caption{Same as Figure~\ref{fig:fmnist} but for OOD images with pixels drawn iid from $N(0,1)$.}
        \label{fig:app_fminst_rand}
\end{figure}

\begin{figure}[tb]
  \centering
  \includegraphics[width=0.48\linewidth]{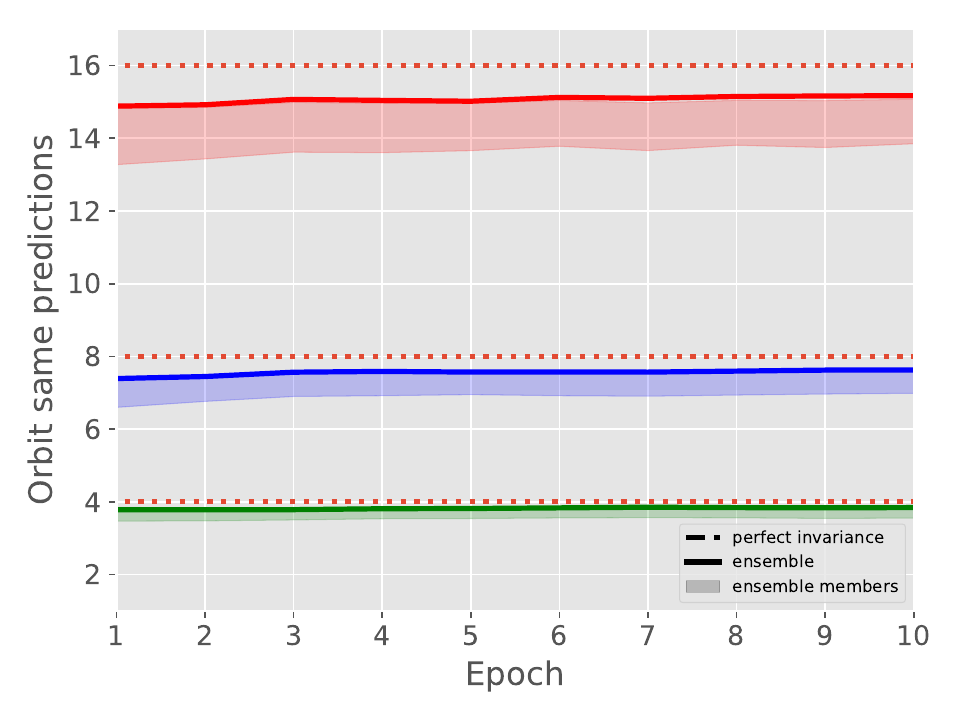}
  \hfill
  \includegraphics[width=0.48\linewidth]{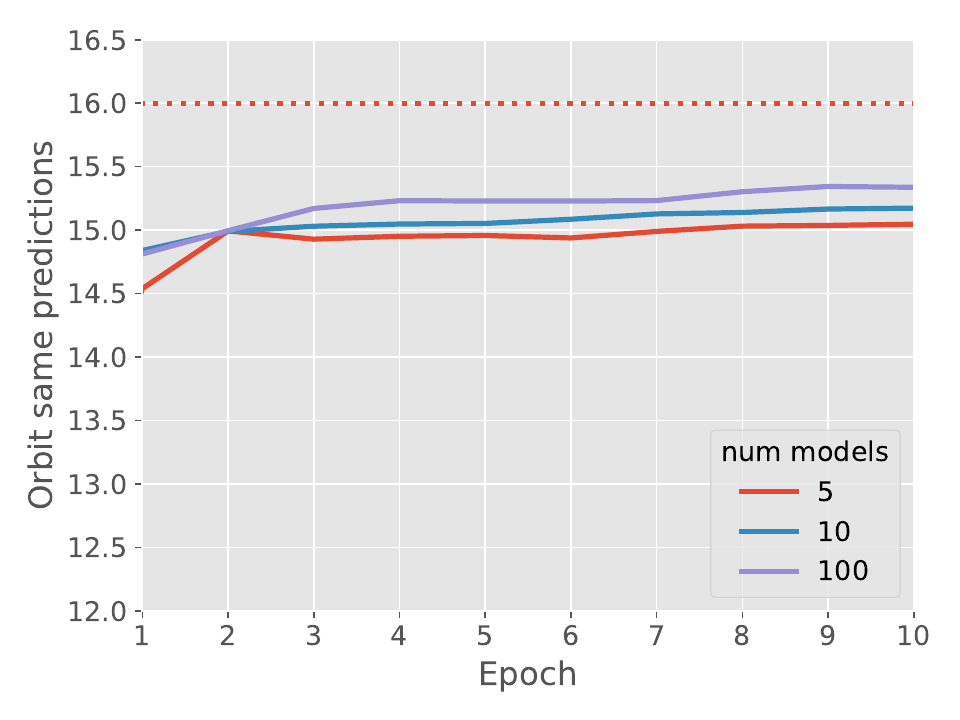}
  \caption{Same as Figure~\ref{fig:fmnist} but for FMNIST, i.e., in-distribution data.}
        \label{fig:app_fminst_fminst}
\end{figure}

\begin{figure}[tb]
  \centering
  \includegraphics[width=0.48\linewidth]{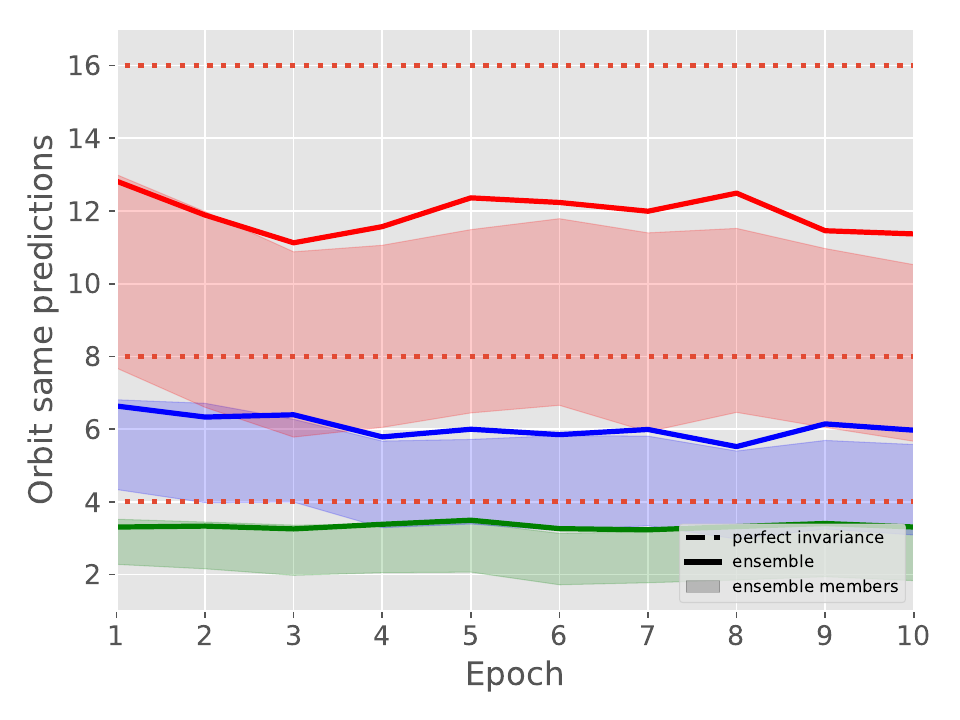}
  \hfill
  \includegraphics[width=0.48\linewidth]{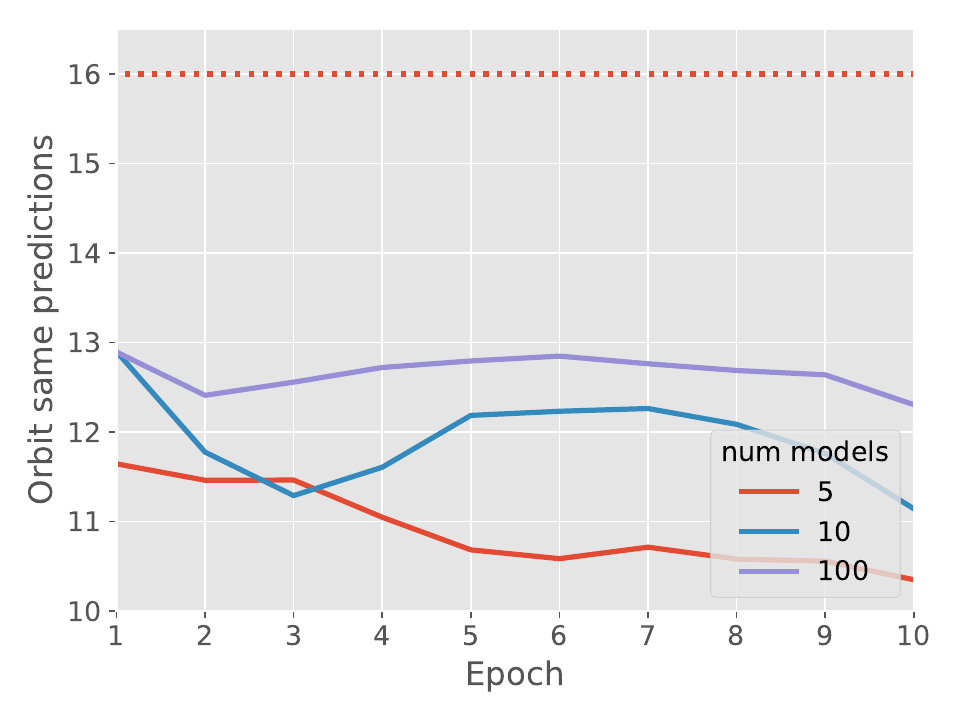}
 \caption{Same as Figure~\ref{fig:fmnist} but for rescaled and greyscaled CIFAR10 OOD data.}
        \label{fig:app_fminst_cifar10}
\end{figure}

\begin{figure}[tb]
  \centering
  \includegraphics[width=0.5\linewidth]{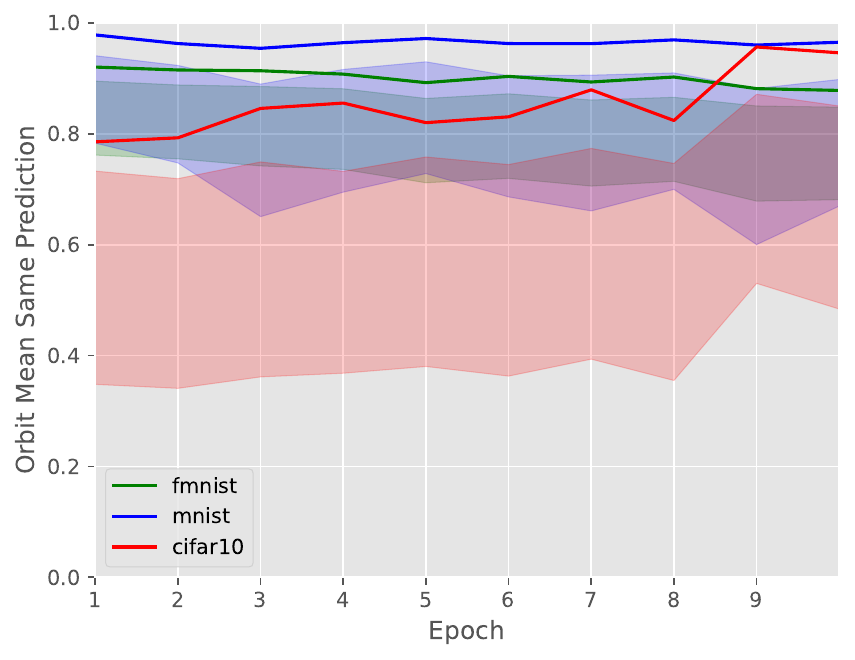}
  \caption{Equivariance extends to $SO(2)$ symmetry. Percentage of randomly sampled rotations that leave the prediction invariant is reported. Data augmentation with group order 16 is used.}
  \label{fig:so2_app}
\end{figure}

\section{Cross Product}
\begin{figure}[tb]
  \centering
  \includegraphics[width=0.8\textwidth]{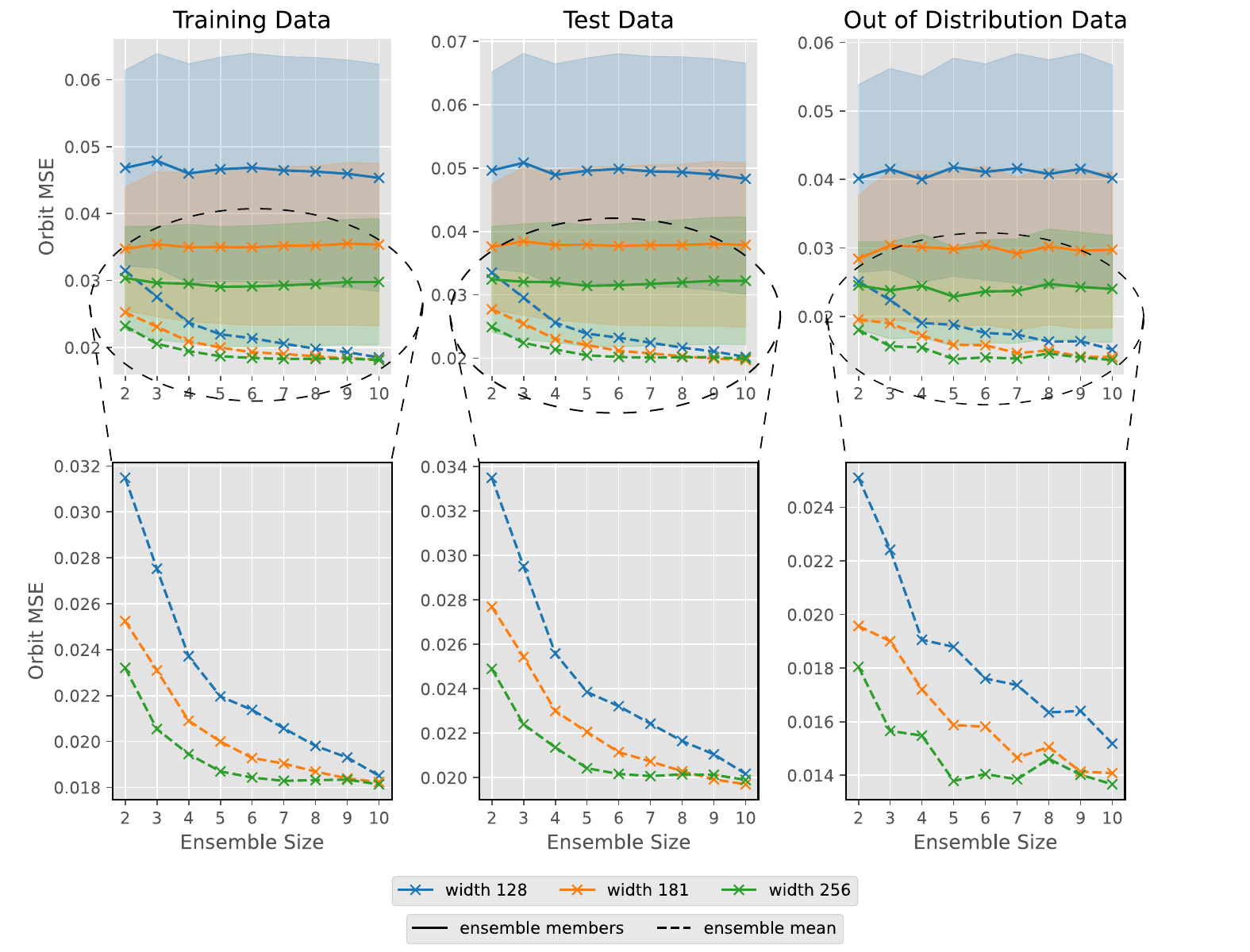}
  \caption{Emerging equivariance of ensembles predicting the cross-product. Plotted is the MSE of predictions across a random 100-element subset of the symmetry orbit of SO(3) versus ensemble size. Solid lines refer to the orbit MSE for individual ensemble members with shaded regions corresponding to $\pm$ one standard deviation, dashed lines refer to the ensemble prediction. Shown are evaluations on the training- (left), test- (middle) and out of distribution data (right). The lower row shows zoom-ins on the ensemble predictions.}
  \label{fig:cross_products}
\end{figure}

\paragraph{Training}
We train ensembles of two hidden-layer fully-connected networks to predict the cross-product $x\times y$ in $\mathbb{R}^3$ given two vectors $x$ and $y$. This task is equivariant with respect to rotations $R\in \mathrm{SO}(3)$,
\begin{align}
    Rx\times Ry=R(x\times y)\,.
\end{align}
The training data consists of 100 vector pairs with components sampled from $\mathcal{N}(0,1)$, the validation data consists of 1000 such pairs. For out of distribution data, we sample from a Poisson distribution with mean 0.5. We train using 10-fold data augmentation, i.e.\ we sample 10 rotation matrices from SO(3) and rotate the training data with these matrices, resulting in 1000 training vector pairs. We train for 50 epochs using the Adam optimizer and reach validation RMSEs of about 0.3 with exact performance depending on layer width and ensemble size.

\paragraph{Orbit MSE}
To evaluate how equivariant the ensembles trained with data augmentation are on a given dataset, we sample 100 rotation matrices from SO(3) and augment each input vector pair with their 100 rotated versions. Then, we predict the cross products on this enlarged dataset and rotate the predicted vectors back using the inverse rotations. Finally, we measure the MSE across the 100 back-rotated predictions against the unrotated prediction. The orbit MSE is averaged over the last five epochs.

 The results of our experiments on the cross-product are shown in Figure~\ref{fig:cross_products}. As above, we evaluate the orbit MSE on each ensemble member individually (solid lines and shaded region corresponding to $\pm$ one standard deviation) and for the ensemble output (dashed lines). This is true on training-, test- and out of distribution data. Also in this equivariant task is the ensemble mean about an order of magnitude more equivariant than the ensemble members. As expected from our theory, the ensemble becomes more equivariant for larger ensembles and wider networks.

\section{Histological Slices}
\label{app:hist_data}
\begin{figure}[tb]
  \centering
  \includegraphics[width=0.48\textwidth]{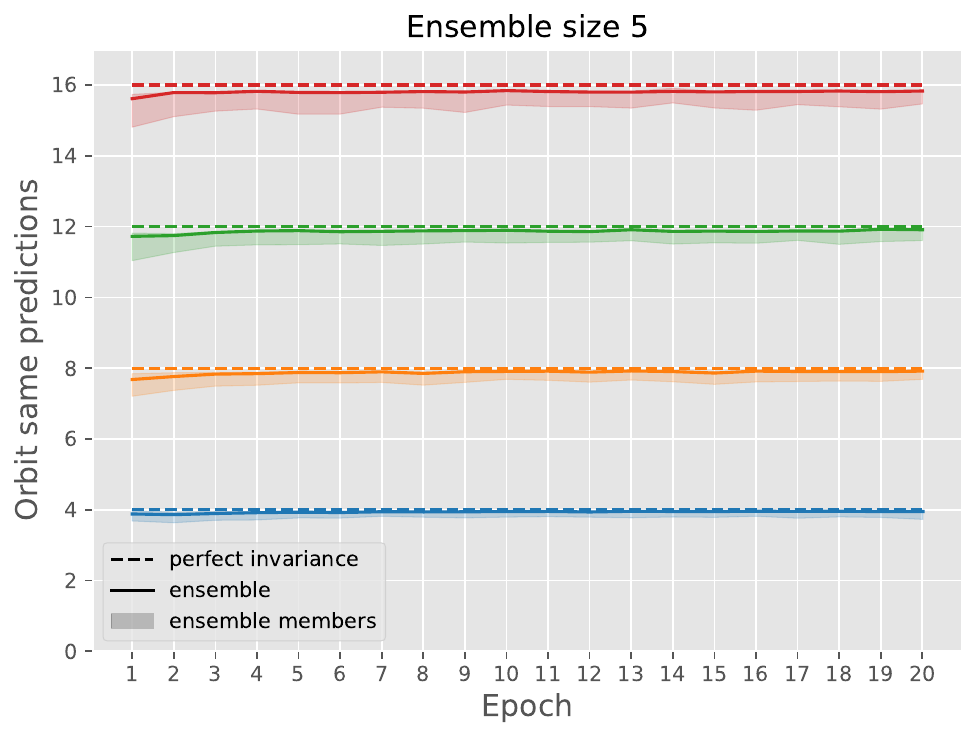}
  \hfill
  \includegraphics[width=0.48\textwidth]{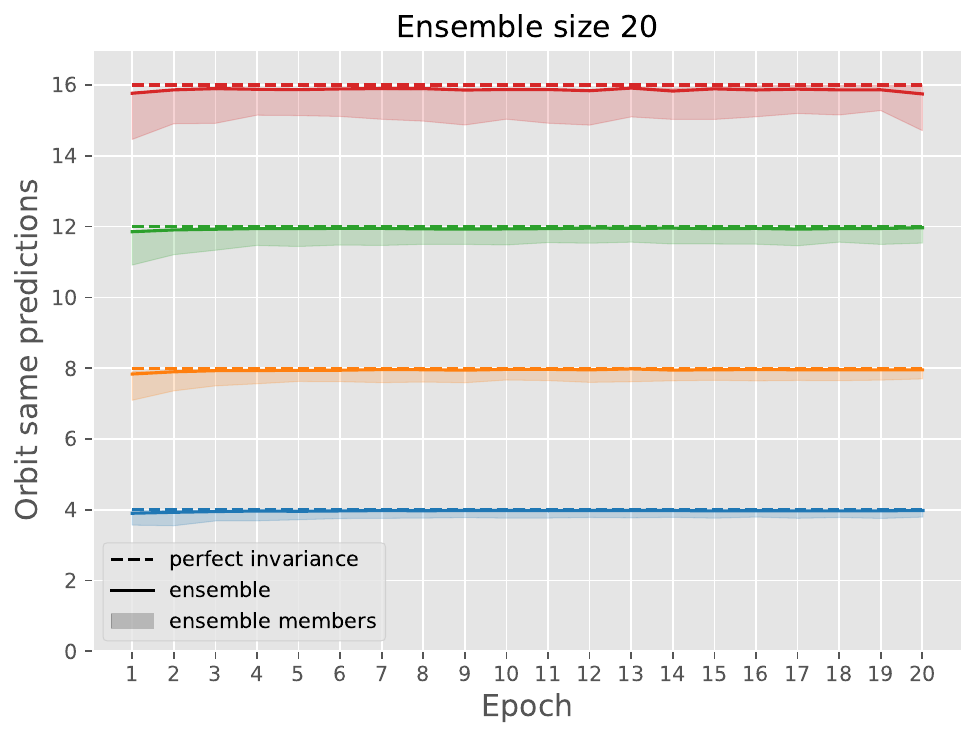}
  \caption{Ensemble invariance on validation data for ensembles trained on histological data. Number of validation samples with the same prediction across a symmetry orbit for group orders 4 (blue), 8 (orange), 12 (green) and 16 (red) versus training epoch for ensemble sizes 5 (left) and 20 (right). The ensemble predictions (solid line) are more invariant than the ensemble members (shaded region corresponding to 25\textsuperscript{th} to 75\textsuperscript{th} percentile of ensemble members). The effect is larger for ensemble size 20 (right).}
  \label{fig:hist_val_plots}
\end{figure}
\begin{figure}[tb]
  \centering
  \includegraphics[width=0.48\textwidth]{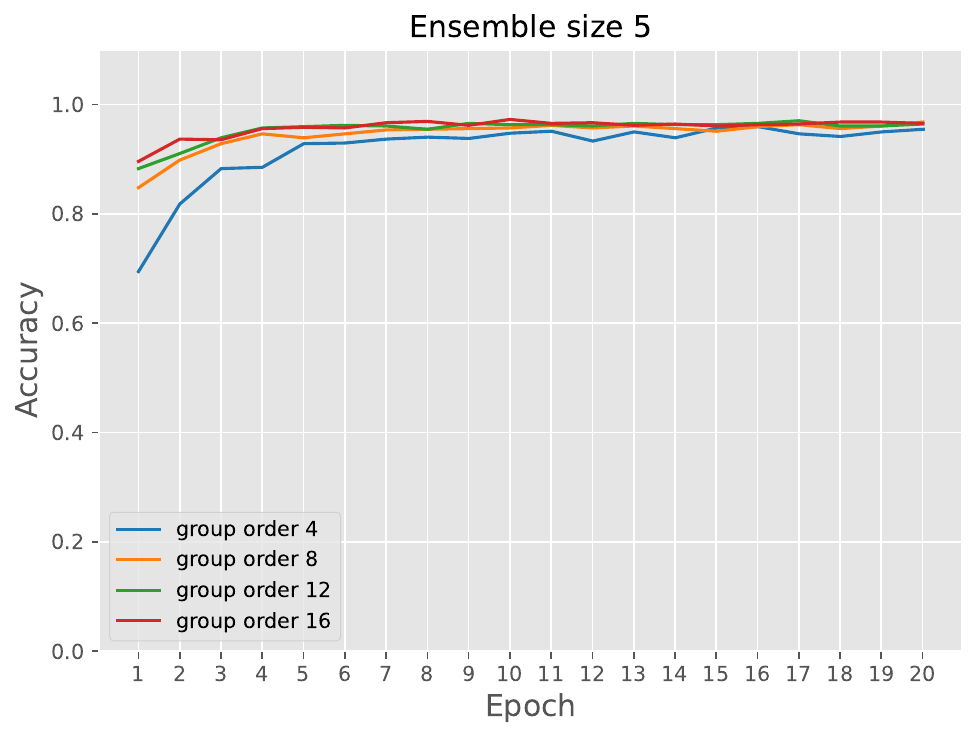}
  \hfill
  \includegraphics[width=0.48\textwidth]{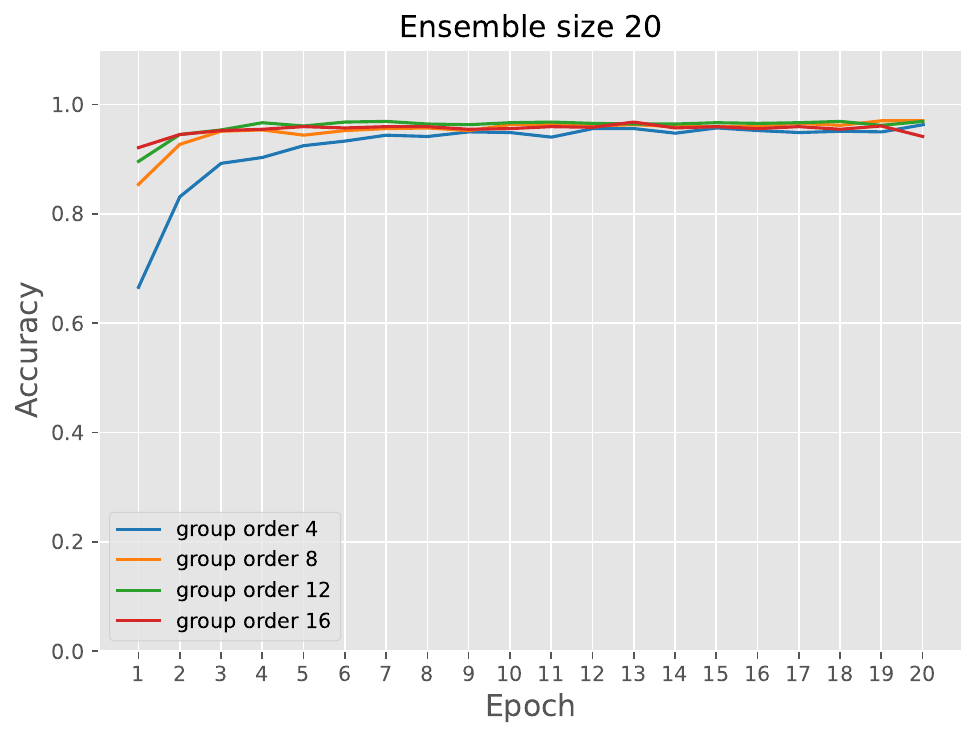}
  \caption{Validation accuracy versus training time for ensemble of size 5 (left) and 20 (right) trained on histological data.}
  \label{fig:hist_val_acc}
\end{figure}

\begin{figure}[t!]
  \centering
  \includegraphics[width=0.5\textwidth]{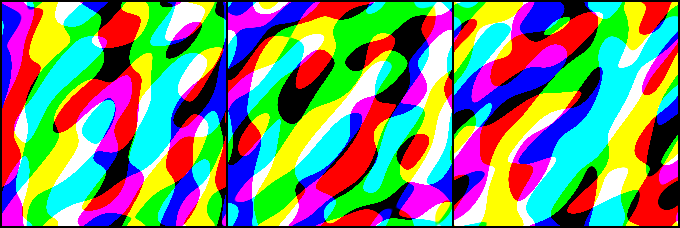}
  \caption{Three OOD data samples for the histology ensemble.}
  \label{fig:hist_ood_samples}
\end{figure}

\paragraph{Training}
The NCT-CRC-HE-100K dataset~\cite{hist_dataset} comprises 100k stained histological images in nine classes. In order to make the task more challenging, we only use 10k randomly selected samples, train on 11/12\textsuperscript{th} of this subset and validate on the remaining 1/12\textsuperscript{th}. We trained ensembles of CNNs with six convolutional layers of kernel size $3$ and $6$, $16$, $26$, $36$, $46$ and $56$ output channels, followed by a kernel size $2$, stride $2$ max pooling operation and three fully connected layers of $120$, $84$ and $9$ output channels. The models had 123k parameters each. We trained the ensembles with the Adam optimizer using a learning rate of $0.001$ on batches of size $16$. In our training setup, ensemble members reach a validation accuracy of about 96\% after 20 epochs, cf.~Figure~\ref{fig:hist_val_acc}.

\paragraph{Invariance on in-distribution data}
As for our experiments on FashionMNIST, we verify that the ensemble is more invariant as a function of its input than the ensemble members. On training- and validation data this is to be expected since the ensemble predictions have a higher accuracy than the predictions of individual ensemble members. The invariance results on validation data are depicted in Figure~\ref{fig:hist_val_plots}.

\paragraph{OOD data}
In order to arrive at a sample of OOD data on which the network makes non-constant predictions, we optimize the input of the untrained ensemble using the Adam optimizer to yield predictions of high confidence ($>99\%$), starting from 100 random normalized images for each class. We optimize only the $5\times 5$ lowest frequencies in the Fourier domain to obtain samples which can be rotated without large interpolation losses, yielding samples as depicted in Figure~\ref{fig:hist_ood_samples}.
